%% file: perrault2019finding.tex
\documentclass[twoside]{article}
\usepackage[accepted]{aistats2019}
\usepackage[utf8]{inputenc} 
\usepackage[T1]{fontenc}    
\usepackage{xcolor} 
\definecolor{blued}{RGB}{70,197,221}
\definecolor{pearOne}{HTML}{2C3E50}
\definecolor{pearTwo}{HTML}{A9CF54}
\definecolor{pearTwoT}{HTML}{C2895B}
\definecolor{pearThree}{HTML}{E74C3C}
\colorlet{titleTh}{pearOne}
\colorlet{bull}{pearTwo}
\definecolor{pearcomp}{HTML}{B97E29}
\definecolor{pearFour}{HTML}{588F27}
\definecolor{pearFith}{HTML}{ECF0F1}
\definecolor{pearDark}{HTML}{2980B9}
\definecolor{pearDarker}{HTML}{1D2DEC}
\usepackage{hyperref}       
\hypersetup{
	colorlinks,
	citecolor=pearDark,
	linkcolor=pearThree,
	breaklinks=true,
	urlcolor=pearDarker}

\usepackage{url}            
\usepackage{booktabs}       
\usepackage{amsfonts}       
\usepackage{nicefrac}       
\usepackage{microtype}      
\usepackage{thmtools} 
\usepackage{amsmath, amsthm, thm-restate}
\usepackage{graphicx}
\allowdisplaybreaks
\newcounter{scratchcounter}

\usepackage[round]{natbib}

\usepackage{tikz}
\usepackage{algorithm}
\usepackage{algorithmic}
\usepackage[colorinlistoftodos, textwidth=22mm, shadow]{todonotes}

\usepackage{array}

\usepackage{xspace}

\usepackage{enumitem}
\usepackage{amssymb}
\usepackage{diagbox} 

\input{misosymbols}

\newcommand{\tableq}[1]{{$\!\begin{aligned} 
               \centering #1
                \end{aligned}$}}

\newcommand{\tabtab}[1]{\begin{tabular}{c}#1\end{tabular}}

\newcommand{\prece}{1|prec|\sum w_jC_j}
\newcommand{\arms}{n}
\newcommand{\hider}{w^\star}
\newcommand{\cost}{c^\star}
\newcommand{\hiderm}{w}
\newcommand{\costm}{c}
\newcommand{\SnS}{sequential search-and-stop\xspace}
\newcommand{\DAG}{\mathcal{G}}
\newcommand*\rd[1]{{#1}}
\newcommand{\rcost}{C}
\newcommand{\costvector}{{\bf \cost}}
\newcommand{\costvectorm}{{\bf \costm}}
\newcommand{\rcostvector}{{\bf \rcost}}
\newcommand{\rhider}{W}
\newcommand{\hidervector}{{\bf\hider}}
\newcommand{\hidervectorm}{{\bf\hiderm}}
\newcommand{\rhidervector}{ {\bf\rhider} }
\newcommand{\distribution}{\mathcal{D}}
\newcommand{\search}{\mathcal{S}}
\newcommand*\shrink[1]{\langle #1\rangle}
\newcommand{\sche}{\textsc{Scheduling}\xspace}
\newcommand{\D}{d}
\newcommand{\J}{J}
\newcommand{\Jp}{J^+}
\newcommand{\history}{\mathcal{H}}
\newcommand{\rbs}{{\bf \rd{s}}}
\newcommand{\rs}{{\rd{s}}}
\newcommand{\RB}{R}
\newcommand{\oracle}{\textsc{Oracle}\xspace}
\newcommand*\Oracle[1]{\textsc{Oracle}\pa{#1}}
\newcommand*\perf[2]{#1[#2]}
\newcommand*\obj[3]{#1\pa{#2;#3}}
\newcommand*\counter[3]{\rd{N}_{#3,#1,#2}}
\newcommand*\meanc[2]{\bar\costm_{#1,#2}}
\newcommand*\meanw[2]{\bar\hiderm_{#1,#2}}

\newcommand*\UCBc[2]{\costm_{#1,#2}}

\newcommand*\UCBw[2]{\hiderm_{#1,#2}}
\newcommand*\vUCBc[1]{ {\costvectorm}_{#1}}
\newcommand*\vUCBw[1]{ {\hidervectorm}_{#1}}

\newcommand*\gapss[1]{\Delta\pa{#1}}
\newcommand{\lip}{\Lambda}
\newcommand{\KL}[2]{\text{KL}\pa{#1\Vert #2}}
\newcommand{\kl}[2]{\text{kl}\pa{#1, #2}}
\newcommand{\sizecorr}[1]{\makebox[0cm]{\phantom{$\displaystyle #1$}}}

\usepackage{authblk}

\begin{document}

%

%

\twocolumn[

\aistatstitle{Finding the bandit in a graph: Sequential search-and-stop}

\aistatsauthor{Pierre Perrault \And Vianney Perchet \And  Michal Valko }

\aistatsaddress{ SequeL team, INRIA Lille\\ CMLA, ENS Paris-Saclay\\ 
\texttt{\small pierre.perrault@inria.fr} 
\And  CMLA, ENS Paris-Saclay \\ Criteo Research\\ \texttt{\small vianney.perchet@normalesup.org} 
\And SequeL team\\ INRIA Lille -- Nord Europe\\\texttt{\small michal.valko@inria.fr} 
} ]
%

%
\begin{abstract}
We consider the problem where an agent wants to find a hidden object that is randomly located in some vertex of a directed acyclic graph (DAG) according to a fixed but possibly unknown distribution. The agent can only examine vertices whose in-neighbors have already been examined. 
In this paper, we address a \emph{learning} setting where we allow the agent to stop before having found the object and restart searching on a new independent instance of the same problem. Our goal is to maximize the total number of hidden objects found given a time budget. The agent can thus skip an instance after realizing that it would spend too much time on~it. 
Our contributions are both to the \emph{search theory} and \emph{multi-armed bandits}. If the distribution is known, we provide a quasi-optimal and efficient stationary strategy. 
If the distribution is unknown, we  additionally show how to sequentially approximate it and, at the same time, act near-optimally  
in order to collect as many hidden objects as possible. 
\end{abstract}

\section{Introduction}
\label{sec:intro}
We study the setting where an object, called hider, is randomly located in one vertex of a directed acyclic graph (DAG), and where an agent wants to find it by sequentially selecting  vertices one by one, and examining them at a (possibly random) cost. The agent has a strong constraint: its search must respect \emph{precedence constraints} imposed by the DAG, i.e., a vertex can be examined only if \emph{all} its in-neighbors have already been examined. The goal of the agent is to minimize the expected total search cost incurred before finding the hider.
This setting is a type of \emph{single machine scheduling} problem \citep{lin2015}, where a set of $\arms$ jobs $[\arms]\triangleq\sset{1,\dots,\arms}$ have to be processed on a single machine that can process at most one job at a time.
Once a job processing is started, it must continue without interruption until the processing is complete.
Each job~$j$ has a cost~$\costm_j$ representing its processing time, and a weight $\hiderm_j$ representing its importance. In our context, $\hiderm_j$ is the probability that $j$ contains the hider. The aim is to find a schedule (i.e., a permutation of jobs) that minimizes the total weighted completion time while respecting precedence constraints\footnote{The standard scheduling notation \citep{Graham1979} denotes this setting  as $\prece$.}. 
The setting was
already shown to be NP-hard 
\citep{Lawler1978,Lenstra1978}. On the positive side, several polynomial-time $\alpha$-approximations exist, depending on the assumption we take on the DAG (see e.g., the recent survey of \citealp{Prot2017}). For instance, the case of $\alpha=2$ can be dealt without any additional assumption. On the other hand, there is an exact $\OO(\arms\log \arms)$-time algorithm when the partially ordered set (poset) defined by the DAG is a series-parallel order \citep{Lawler1978}. More generally, when the poset has fractional
dimension of at most $f$, there is a polynomial-time approximation with $\alpha=2-2/f$ \citep{Ambuhl2011}.
In this work, we assume the DAG is such that an exact polynomial-time algorithm is available. We denote this algorithm as $\sche$. For example, this is true for two-dimensional posets \citep{Ambuhl2009}.

The problem is also well known in \emph{search theory} \citep{stone1976theory,Fokkink2016}, one of the disciplines originating from \emph{operations research}.
Since in our case, the search space is a DAG,  we fall within the network search setting  \citep{kikuta1994initial,gal2001optimality,evans2013static}. When the DAG is an out-tree, the problem reduces to the \emph{expanding search} problem introduced by \citet{alpern2013mining}.

The case of unknown distribution of the hider is usually studied within the field of \emph{search games}, i.e., a zero-sum game where the agent picks the search and plays against the hider
with search cost as payoff
\citep{alpern2006theory,alpern2013search,hohzaki2016search}. 
In our work, we deal with an unknown hider distribution by extending the stochastic setting to the {sequential case}, where at each round $t$, the agent faces a new, independent instance of the problem. The challenge is the need to learn the distribution through repeated interactions with the environment.
Each instance, the agent has to perform a search based on the instances observed during the previous rounds. 
Furthermore, contrary to the typical {search} setting, the agent can additionally decide whether it wishes to abandon the search on the current instance and start a new one in the next round, even if the hider was not found. The goal of the agent is to collect as many hiders as possible, using a fixed budget $B$.
This may be particularly useful,  when the remaining vertices have large costs and it 
would not be cost-effective to examine them. 

As a result, the hider may not be found in each round and the agent has to make a trade-off between exhaustive searches, which lead to a good estimation (\emph{exploration}) and efficient searches, which leads to a good benefit/cost ratio (\emph{exploitation}). The sequential \emph{exploration-exploitation} trade-off is well studied
in multi-armed bandits (\citealp{cesa-bianchi2006prediction,lattimore2019bandit}) and has been applied to many fields including mechanism design \citep{mohri2014optimal}, search advertising \citep{tran2014efficient} and personalized recommendation \citep{li2010contextual}. Since several vertices can be visited within each round, our setting can be seen as an instance of {stochastic combinatorial semi-bandits} \citep{cesa-bianchi2006prediction,cesa-bianchi2012combinatorial,gai2012combinatorial,gopalan2014thompson,kveton2015tight,combes2015combinatorial,wang2017improving,valko2016bandits}. For this reason, we refer to a vertex $j\in [\arms]$ as an \emph{arm}. We shall see, however, that
this specific semi-bandit problem is challenging. In particular, the agent pays a \emph{non-linear} search cost at each round (with respect to the selected combinatorial action), that additionally depends on the ordering.
Moreover, due to the budget constraint, it is also an instance of \emph{budgeted bandits}, also known as \emph{bandits with knapsacks} \citep{badanidiyuru2013bandits}, in the case of single resource and infinite horizon. We thus evaluate the performance of a learning policy with the (common) notion of \emph{expected (budgeted) regret}. It measures the expected difference, in terms of cumulative reward collected within the budget constraint $B$, between the learning policy and an \emph{oracle} policy that knows a priori the exact parameters of the problem. Budgeted combinatorial semi-bandits have been already studied by \citet{Sankararaman2017} for several resources, but with a finite horizon. Moreover, their algorithm is efficient only for some specific combinatorial structures (such as matroids). 
The structure of constraints in \SnS is in general more complex.
\paragraph{Motivation}
There are several motivations behind this setting. One example is the \emph{decision-theoretic
troubleshooting} problem of giving a diagnosis
for several devices having a malfunctioning component and arriving sequentially to the agent.
In many \emph{troubleshooting} applications, we additionally face
precedence constraints. These restrictions are imposed to the agent as the ordering of component tests, see e.g., \citealp{Jensen2001}. Moreover, allowing the agent to \emph{stop} gives a new alternative to the so-called \emph{service call} \citep{Heckerman1995,Jensen2001} in order to deal with non-cost-effective vertices: Instead of giving a high cost to an extra action that will automatically find the fault in the device, we give it a zero cost, but do not reward such diagnostic failure. This way, we do not need to 
estimate any \emph{call-service} cost. This alternative is used, for example, when a new device is sent to the user if the diagnostic fails, with a cost that depends on a disutility for the user: loss of personal data, device reconfiguration, etc. Maximizing the number of hiders found is then analogous to maximizing the number of successful diagnoses.

Another example comes from online advertisement. There are several different actions that might generate a conversion from a user, such as sending one or several emails, displaying one or several ads on a website, buying keywords on search engines, etc. We assume that some precedence constraints are imposed between actions and that
a conversion will occur if some sequence of actions is made, for instance, first, display an ad, then send the first email, and finally the second one. As a consequence, the conversion is ``hidden'', the precedence constraints restrict our access to it, and the agent aims at finding it. However, for some users, finding the correct sequence might be too expensive and it might be more interesting to abandon that specific user to focus on more promising ones.

\paragraph{Related settings} Finally, there are several settings related to ours. One of them is {stochastic probing} \citep{Gupta2013}, which differs in the fact that each arm can contain a hider, {independently from each other}. Another one is the machine learning framework of {optimal discovery} \citep{bubeck2013optimal}. 
\paragraph{Our contributions}
One of our main contributions is a stationary \emph{offline policy} (i.e., an algorithm that solves the problem when the distribution is known), for which we prove the approximation guarantees and adapt it in order to fit the online problem. In particular, we prove that it is quasi-optimal  and use $\sche$ to prove its computational efficiency. 
Next, we provide a solution when the distribution is unknown to the agent, based on combinatorial upper confidence bounds (\CUCB) algorithm from \citet{chen2015combinatorial}, and \UCB-variance (\UCBV) of \citet{audibert2009}.       
Dealing with variance estimates allows us to sharp the bound on the expected regret, improving the overall dependence on the dimension $\arms$ compared to the simple use of \CUCB.
We also propose a new method (that can be of independent interest) to avoid the typical ${\nicefrac{1}{\costm_{\min}^2}}$ term in the expected regret bound \citep{tran-thang2012knapsack,ding2013multi-armed,pmlr-v45-Xia15,xia2016budgeted,article}, where $\costm_{\min}$ is the minimal expected search cost paid over a single round. 


\section{Background}  
\label{sec:back}
In this paper, we typeset vectors in bold and indicate components with indices, i.e., $\mathbf{a}=(a_i)_{i\in [\arms]} \in \R^n$. 
We formalize in this section the setting we consider. We denote a finite DAG by $\DAG\triangleq\pa{[\arms],\edgeset}$, where $[\arms]$ is its set of vertices, or arms, and $\edgeset$ is its set of directed edges. 
For more generality, we assume arm costs are random and mutually independent. We denote $\rcost_j\in[0,1]$, with expectation $\cost_j\triangleq\EE{\rcost_j} >0$, the cost of arm $j$. We thus have $\costvector=\EE{\rcostvector}\in (0,1]^\arms$. 
We also assume that one specific vertex, called \emph{hider}, is chosen at random, independently from $\rcostvector$, accordingly to some fixed categorical (or multivariate Bernoulli) distribution parameterized by  
 vector $\hidervector$ satisfying\footnote{i.e., $\hidervector$ belongs to the simplex of $\R^\arms$} $\sum_{i=1}^\arms \hider_i =1$ and $\hider_i \in [0,1]$. Notice that $\rhidervector\sim \Bernoulli(\hidervector)$ if, given $i \in [\arms]$ and with probability $\hider_i$,  $\rhider_i=1$ and $\rhider_j=0$ for all $j\neq i$. We also use $\distribution$ to denote the joint distribution of $(\rcostvector,\rhidervector)$. 
 
 Let
$\be_i\in \R^\arms$ denote the $i^{\rm th}$  canonical  unit  vector.
 For an (ordered) subset $A$ of $[\arms]$, we denote by $A^c$ the complementary of $A$ in $[\arms]$ and $\abs{A}$ its cardinality. 
 The \emph{incidence vector} of $A$ is
$\be_A\triangleq\sum_{i \in A}\be_i.$
The above definition allows representing a subset of $[\arms]$ as an element of $\sset{0,1}^\arms\!\!.$
Let $\DAG\shrink{A}$ be the sub-DAG in $\DAG$ induced by $A$, i.e., the DAG with $A$ as vertex set, and with $(i,j)$ an arc in $\DAG\shrink{A}$ if and only if $(i,j)\in\edgeset$. 
 We call \emph{support} of an ordered arm set $\ba=(a_1,\dots,a_k)$ the corresponding non-ordered set. 
 If $\bx,\by\in \R^\arms$, we write $\bx\geq\by$ (resp., $\bx\leq\by$) if $\bx-\by\in \R_+^\arms$ (resp., $\by-\bx\in \R_+^\arms$).
 We let $
\mathbf{a}[j]\triangleq(a_1,a_2,\ldots,a_j)$ for $j\leq \abs{\ba}$. In addition, we let $\mathbf{a}[\rhidervector]\triangleq \mathbf{a}[j]$ if there is $j$ such that $\rhider_{a_j}=1$, and $\mathbf{a}[\rhidervector]\triangleq \mathbf{a}$ otherwise. For two disjoint ordered arm sets $\mathbf{a}$ and $\mathbf{b}$, we let $\mathbf{ab}=(a_1,a_2,\ldots,a_{\abs{a}},b_1,b_2,\ldots,b_{\abs{b}})$ be the concatenation of $\mathbf{a}$ and $\mathbf{b}$.

 We assume that $\DAG$ allows a polynomial-time algorithm (w.r.t.\,$\arms$), that takes some parameters $\hidervectorm,\costvectorm\in \R_+^\arms$, and outputs $\bs=\sche\pa{\hidervectorm,\costvectorm,\DAG}$  
 minimizing 
 \[\obj{\D}{\bs}{\hidervectorm,\costvectorm}\triangleq\sum_{i=1}^{\abs{\bs}}\hiderm_{s_i}  \be_{\mathbf{s}[i]}\transpose\costvectorm =\sum_{i=1}^{\abs{\bs}}\hiderm_{s_i} \sum_{j=1}^i \costm_{s_j}\] over linear extensions\footnote{A linear extension of a poset is a total ordering consistent with the poset, i.e., if $a$ is before $b$ in the poset, then the 
 same has to be true for its  linear extension.}
  $\bs=\pa{s_1,\dots,s_\arms}$ of the poset defined by $\DAG$ (that we call $\DAG$-linear extensions). Notice that
 $\D(\bs)\triangleq\obj{\D}{\bs}{\hidervector,\costvector}$ represents the expected cost $ \EE{\be_{\bs[\rhidervector]}\transpose\rcostvector}$ to pay for finding the hider with the $\DAG$-linear extension $\bs$, i.e., by searching arm $s_1$ first and paying $\rcost_{s_1}$, then $s_2$ by paying $\rcost_{s_2}$ in case $\rhider_{s_1}=0$, and so on until the hider is found, i.e., the last arm $i$ searched is such that $\rhider_{i}=1$.  

We define a \emph{search} in $\DAG$ as an ordering $\mathbf{s}=(s_1,\ldots,s_k)$ of different arms such that for all $i\in [k]$, predecessors of $s_i$ in $\DAG$ are included in $\sset{s_1,\dots,s_{i-1}}$, i.e., a search is a prefix of a $\DAG$-linear extension. We denote by  $\search_\DAG$ (or simply $\search$) the set of searches in $\DAG$. 
Search supports are called \emph{initial sets}. 

\subsection{Protocol}
Our search setting is \emph{sequential}. We consider an \textsl{agent}, also called a \textsl{learning algorithm} or a \textsl{policy} that knows~$\DAG$ but that does not know $\distribution$.  
At each round $t$, an independent sample $(\rcostvector_t,\rhidervector_t)$ is drawn from $\distribution$. The aim of the agent is to search the hider (i.e., the arm $i$ such that $\rhider_{i,t}=1$) by constructing a \textsl{search} on $\DAG$. Since the hider may be located at some arm that does not belong to the search, it is not necessarily found over each round. 

The search to be used by the agent can be chosen based on all its previous observations, i.e., all the costs of explored vertices (and only those) and all the locations where the hider has been found or not. Obviously, the search cannot use the non-observed quantities.
For example, the agent may estimate $\hidervector$ and $\costvector$ in order to choose the search accordingly. 
Each time an arm~$j$ is searched, the feedback $\rhider_{j,t}$ and $\rcost_{j,t}$ is given to the agent.  
Since several arms can be searched over one round, this problem falls into the family of
\emph{stochastic combinatorial semi-bandits}.
The agent can keep searching until its budget, $B$, runs out.
$B$ is a positive number and does not need to be known to the agent in advance.
The agent wants to maximize the overall number of hiders found under the budget constraint.

The setting described above allows the agent to modify its behavior depending on the feedback it received during the current round. However, by independence assumption between random variables, the only 
feedback susceptible to modify the search the agent chose at the beginning of a round~$t$ is the observation of $\rhider_{i,t}=1$ for some arm $i$. Even if nothing prevents the agent from continuing ``searching'' some arms after having seen such an event, it would not increase the number of hiders found (there is no more hider to find), while this would still decrease the remaining budget, and therefore it would have a pure exploratory purpose. Knowing this, an oracle policy that knows exactly $\distribution$ thus \emph{selects} a search $\bs$ at the beginning of round $t$, and then \emph{performs} the search that follows~$\bs$ until either $\rhider_{i,t}=1$ is observed or $\bs$ is exhausted (i.e., no arms are left in $\bs$).
Therefore, 
the performed search is in fact $\bs[\rhidervector_t]$. We thus 
restrict ourselves to agents that select a search $\bs$ at the beginning of each round $t$ and then performs $\bs[\rhidervector_t]$ over this round. As a consequence, the selected search $\bs$ is computed based on observations collected during previous rounds $t-1,t-2,\dots$, denoted~$\history_t$, that we refer to as \emph{history}. 
   
Following \citet{stone1976theory}, we refer to our problem as \emph{sequential search-and-stop}. We now detail the overall objective for this problem: The agent wants to follow a policy $\pi$, that selects a search $\rbs_t$ at round $t$ (this choice can be random as it  may depend on the past observations $\history_t$, as well as possible randomness from the algorithm), while
maximizing the expected overall reward
  \begin{equation*}F_B({\pi})\triangleq\EE{\sum_{t=1}^{\rd{\tau}_B-1} \be_{\rbs_t[\rhidervector_t]}\transpose \rhidervector_t }=\EE{\sum_{t=1}^{\rd{\tau}_B-1} \sum_{i\in \rbs_t[\rhidervector_t]}\rhider_{i,t}},\end{equation*}
  where $\rd{\tau}_B$ is the random round at which the remaining budget becomes negative: In particular, we have that if $\rd{B}_t\triangleq B-\sum_{u=1}^t\be_{\rbs_u[\rhidervector_u]}\transpose\rcostvector_t$, then $\rd{B}_{\rd{\tau}_B -1}\geq 0$ and $\rd{B}_{\rd{\tau}_B}<0$.
 We evaluate the
 performance of a policy using the \emph{expected (budgeted) regret} with respect to $F_B^\star$, the maximum value of $F_B$ (among all possible oracle policies that know $\distribution$ and $B$), defined as
 \begin{equation*}\RB_B({\pi})\triangleq F_B^\star-F_B({\pi}).\end{equation*}
 
\begin{example}
One may wonder if there exist cases where it is interesting for the agent to stop the search earlier.  Consider for instance the simplest non-trivial case with
two arms and no precedence constraint. The costs are deterministically chosen to be $\varepsilon$ and $1$ and the location of the hider is chosen uniformly at random. An
optimal search will always first sample  the arm with $\varepsilon<1$ cost. 
If it also samples the other one, then the hider will be found with an expected cost of
$\varepsilon+\nicefrac12$.
However, if the agent always stops the search after the first arm, and reinitializes on a new instance by doing the same, the overall cost to find one hider is
\[\sum_{t=1}^\infty \left({1\over 2}\right)^{\!t}t\varepsilon={2\varepsilon}<\varepsilon+{1\over 2}\!\CommaBin \quad \text{ for } \varepsilon <{1\over 2}\cdot\]
 Therefore, stopping searches, even if the location of the hider is known, may be better than always trying to find~it.
\end{example} 
\section{Offline oracle}
In this section, we provide an algorithm for \SnS when parameters $\hidervector$ and $\costvector$ are given to the agent. We show that a simple  stationary policy (i.e., the same search $\bs^\star$ is selected at each round) can obtain almost the same expected overall reward as $F_B^\star$. We will denote by \oracle an algorithm that takes $\hidervector$, $\costvector$, and $\DAG$ as input and outputs $\bs^\star$.  
This \emph{offline} oracle will  eventually be used by the agent for the \emph{online} problem, i.e., when parameters are unknown. Indeed, at round $t$, the agent can approximate $\bs^\star$ by the output $\bs_t$ of $\Oracle{\hidervectorm_t,\costvectorm_t,\DAG}$, where $\hidervectorm_t,\costvectorm_t$ can be any guesses/estimates of the true parameters.
Importantly, depending on the estimation followed by the agent, $\hidervectorm_t$ may not stay in the simplex anymore. We will thus build \oracle such that an ``acceptable'' output is given for any input $\pa{\hidervectorm,\costvectorm}\in\pa{\R_+^\arms}^2\!\!.$ 
\subsection{Objective design}
A standard paradigm for designing a stationary approximation of the offline problem in budgeted multi-armed bandits is the following: $\bs^\star$ has to minimize the ratio between the expected cost paid and the expected reward gain, over a single round, selecting $\bs^\star$. 
We thus define, for $\bs\in \search$,
\begin{align*} {\J}\pa{\bs}&\triangleq \EE{\be_{\bs[\rhidervector]}\transpose\rcostvector}\EE{\be_{\bs[\rhidervector]}\transpose\rhidervector}^{-1} \\&={{\D}\pa{\bs}{}+(1-\be_\mathbf{s}\transpose\hidervector)\be_\mathbf{s}\transpose\costvector\over \be_{\mathbf{s}}\transpose\hidervector }\\&=
\sum_{i=1}^{|\mathbf{s}|}{\cost_{s_i}\pa{1-\be_{\bs[i-1]}\transpose\hidervector}\over \be_{\mathbf{s}}\transpose\hidervector}\cdot\end{align*}
Notice that we allow $\J$ to be equal to $+\infty$ (when $\be_{\mathbf{s}}\transpose\hidervector=0$).
We use the convention $\J(\emptyset)=+\infty$, because there is no interest in choosing an empty search for a round.
We  define the optimal values of $J$ on $\search$ as
\[\J^\star\triangleq\min_{\mathbf{s} \in \search}\J(\mathbf{s}),\quad\search^\star\triangleq \argmin_{\mathbf{s} \in \search} \J(\mathbf{s}).\]
\noindent
We now provide guarantees for this stationary policy.
\begin{restatable}{proposition}{restaC}\label{prop:approx}
 If $\pi^\star$ is the offline policy selecting $\bs^\star\in \search^\star$ at each round $t$, then \[\frac{B-\arms}{\J^\star}\leq F_B(\pi^\star)\leq F_B^\star\leq \frac{B+\arms}{\J^\star}\cdot\]
\end{restatable}
\noindent
A proof is given in Appendix~\ref{app:approx} and follows the one provided for Lemma~1 of~\cite{xia2016budgeted}. Intuitively, Proposition~\ref{prop:approx} states that the optimal overall expected reward that can be gained (i.e., the maximum expected number of hiders found) is approximately $B/\J^\star$ (we assume that $B\gg\arms$). This is quite intuitive, since this quantity is actually the ratio between the overall budget and the minimum expected cost paid to find a \emph{single} hider. Indeed, one can consider the related problem of minimizing the overall expected cost paid, over several rounds, to find a single hider. It can be expressed as an infinite-time horizon Markov decision process (MDP) 
with  action space $\search$ and two states:  whether the hider is found (which is the terminal state) or not.
The goal is to choose a strategy $\mathbf{s}_1,\mathbf{s}_2,\ldots, \mathbf{s}_t, \ldots$, minimizing 
\begin{align*}&\cJ\left(\mathbf{s}_1,\mathbf{s}_2,\ldots\right)\triangleq\EE{\sum_{t=1}^{\rd{\tau}}\be_{\rbs_t[\rhidervector_t]}\transpose \rcostvector_t}\\&=\sum_{t=1}^\infty\left( \be_{\mathbf{s}_t}\transpose\hidervector\pa{\sum_{u=1}^{t-1}\be_{\mathbf{s}_u}\transpose}\costvector+\D\pa{\mathbf{s}_t}\right)\prod_{u=1}^{t-1} 
\left(1-{\be_{\mathbf{s}_u}\transpose\hidervector}\right),\end{align*}
where the  stopping time $\rd{\tau}$ is the first round at which the hider is found. 
The Bellman equation is \begin{equation*}\cJ\left(\mathbf{s}_1,\mathbf{s}_2,\ldots\right)=\D(\mathbf{s}_1) +
\pa{1-\be_{\mathbf{s}_1}\transpose\hidervector}\!\left(\be_{\mathbf{s}_1}\transpose\costvector+\cJ(\mathbf{s}_2,\ldots)\right),\end{equation*}
from which we deduce there exists an optimal \emph{stationary} strategy \citep{sutton1998reinforcement} such that $\mathbf{s}_t=\mathbf{s}$ for all $t \in \mathbb{N}^\star$. Therefore, we can minimize $\cJ\pa{\bs,\bs,\dots}=\J(\bs)$ that gives the optimal value of $\J^\star$. 

As we already mentioned, \oracle aims at taking inputs $\pa{\hidervectorm,\costvectorm}\in\pa{\R_+^\arms}^2\!\!.$
The first straightforward way to do is to consider
\[\obj{\J}{\bs}{\hidervectorm,\costvectorm}\triangleq\sum_{i=1}^{|\mathbf{s}|}{\costm_{s_i}\pa{1-\be_{\bs[i-1]}\transpose\hidervectorm}\over \be_{\mathbf{s}}\transpose\hidervectorm}\cdot\] 
However, notice that with the definition above, $\obj{\J}{~\!\!\cdot~\!\!}{\hidervectorm,\costvectorm}$ could output negative values (if $\be_{[\arms]}\transpose\hidervectorm> 1$), which is not desired, because the agent would then be enticed to search arms with a high cost. We thus need to design a non-negative extension of $\J$ to $\pa{\hidervectorm,\costvectorm}\in\pa{\R_+^\arms}^2\!\!.$ One way is to replace $\left(1-\be_{\bs[i-1]}\transpose\hidervectorm\right)$ by $\be_{\pa{\bs[i-1]}^c}\transpose\hidervectorm$, another is to consider $\obj{\J}{\bs}{ \hidervectorm,\costvectorm}^+\!,$ where $x^+\triangleq\max\sset{0,x}$. There is a significant advantage of considering the second way, even if it is less natural than the first one, which is that for $\pa{\hidervectorm,\costvectorm}\in \pa{\R_+^\arms}^2\!,$ \[\obj{\J}{\bs}{\hidervectorm,\costvectorm}^+\leq\obj{\J}{\bs}{\hidervector,\costvector}= \J(\bs),\] if $\hidervectorm\geq \hidervector$ and $\costvectorm\leq \costvector$. 
This property\footnote{Notice this is not exactly a monotonicity property, because we compare to a single point $(\hidervector, \costvector)$. } is known to be useful for analysis of many stochastic combinatorial semi-bandit algorithms (see e.g., \citealp{chen2015combinatorial}).
Thus, we choose for \oracle the minimization of the surrogate $\obj{\J}{~\!\!\cdot~\!\!}{\hidervectorm,\costvectorm}^+\!\!.$
 
\subsection{Algorithm and guarantees}
We now provide \oracle in Algorithm~\ref{algo:oracle} and claim in Theorem~\ref{thm:oracle} that it minimizes $\obj{\J}{~\!\!\cdot~\!\!}{\hidervectorm,\costvectorm}^+$ over $\search$.  
Notice that \oracle needs to call the polynomial-time algorithm $\sche\pa{\hidervectorm,\costvectorm,\DAG}$, that minimizes the objective function $\obj{\D}{\bs}{\hidervectorm,\costvectorm}$ over $\DAG$-linear extensions~$\bs$. Then, Algorithm~\ref{algo:oracle} only computes the maximum value index of a list of size $\arms$ that takes linear time.
To give an intuition, $\bs^\star$ follows the ordering given by $\sche\pa{\hidervectorm,\costvectorm,\DAG}$, and stops at some point when it becomes more interesting to start a fresh new instance. 
\begin{algorithm}[H]
\begin{algorithmic}
\STATE \textbf{Input}: $\hidervectorm,\costvectorm$ and $\DAG$.

\quad $\bs\triangleq\sche\pa{\hidervectorm,\costvectorm,\DAG}$.

\quad $i^\star\triangleq\argmin_{i\in[\arms]}\obj{\J}{\bs[i]}{\hidervectorm,\costvectorm}^+$ (ties may be broken arbitrarily).

\STATE \textbf{Output}: the search $\bs^\star\triangleq\bs[i^\star]$.
\end{algorithmic}
\caption{\oracle}\label{algo:oracle}
\end{algorithm}
\begin{restatable}{theorem}{restaA}
\label{thm:oracle}
For every $(\hidervectorm,\costvectorm)\in \pa{\R^\arms_+}^2\!,$ Algorithm~\ref{algo:oracle} outputs a search minimizing $\obj{\J}{~\!\!\cdot~\!\!}{\hidervectorm,\costvectorm}^+$ over $\search$.
\end{restatable}
We provide a proof of Theorem~\ref{thm:oracle} in  Appendix~\ref{app:sidney}. It mixes known concepts of scheduling theory, such as Sidney decomposition \citep{Sidney1975}, with our new results for our objective function.
\section{Online search-and-stop}
In this section,  we consider  an additional challenge  where   the distribution $\distribution$ is unknown and the agent must deal with it, while minimizing $\RB_B({\pi})$ over sampling policies $\pi$, where $B$ is a fixed budget.
Recall that a policy $\pi$ selects a search $\rbs_t$ at the beginning of round~$t$, using previous observations $\history_t$, and then performs the search $\rbs_t[\rhidervector_t]$ over the round. We treat  the setting as a variant of  stochastic combinatorial semi-bandits   \citep{gai2012combinatorial}.
The feedback received by an agent at round $t$ is random, because it depends on $\rbs_t$. However, unlike in similar settings, it also depends on~$\rhidervector_t$, and thus it is not measurable w.r.t.\,$\history_t$. More precisely, $(\rhider_{i,t},\rcost_{i,t})$ is observed only for arms $i\in\rbs_t[\rhidervector_t]$. 
Notice that since $\rhidervector_t$ is a one-hot vector, the agent can always deduce the value of $\rhider_{i,t}$ for all $i\in \rbs_t$.  
As a consequence, we will maintain two types of counters for all arms $i \in [\arms]$ and all $t\geq 1,$ 
\begin{equation}
\begin{aligned}
\label{counter}\counter{i}{t-1}{\hidervectorm}&\triangleq\sum_{u=1}^{t-1}\II{i\in \rbs_u}
,\\\counter{i}{t-1}{\costvectorm}&\triangleq\sum_{u=1}^{t-1}\II{i\in\perf{\rbs_u}{\rhidervector_u} }.\end{aligned} \end{equation}
We define the corresponding empirical averages\footnote{With the convention $0/0=0$.} as
\begin{equation}
\begin{aligned}
\label{average}
\meanw{i}{t-1}&\triangleq\frac{\sum_{u=1}^{t-1}\II{i\in \rbs_u}\rhider_{i,u}}{\counter{i}{t-1}{\hidervectorm}}
\CommaBin
\\\meanc{i}{t-1}&\triangleq\frac{\sum_{u=1}^{t-1}\II{i\in\perf{\rbs_u}{\rhidervector_u} }\rcost_{i,u}}{\counter{i}{t-1}{\costvectorm}}\cdot
\end{aligned}
\end{equation}
\noindent 
We propose an approach similar to \UCBV of \citet{audibert2009}, based on \CUCB of \citet{chen2015combinatorial}, called \CUCBV, that uses a variance estimation of~$\hidervector$ in addition to the empirical average. Notice that the variance of $\rhider_i$ for an arm $i$ is $\sigma_i^2\triangleq\hider_i(1-\hider_i)$. Furthermore, since $\rhider_i$ is binary, the empirical variance of~$\rhider_i$ after $t$ rounds is $\meanw{i}{t}(1-\meanw{i}{t}).$ For every round~$t$ and every edge $i \in [\arms]$, with the previously defined empirical averages, we use the confidence bounds\footnote{With the convention $x/0=+\infty,~\forall x\geq 0$.} as
\begin{align*}
\UCBc{i}{t}&\triangleq\pa{ \meanc{i}{t-1} - \sqrt{0.5\zeta\log t\over \counter{i}{t-1}{\costvectorm}}}^+\!\!\!\!,\\ 
\UCBw{i}{t}&\triangleq\min\left\{ \meanw{i}{t-1} + \sqrt{2\zeta\meanw{i}{t-1}(1-\meanw{i}{t-1})\log t\over \counter{i}{t-1}{\hidervectorm}}  \right.
\\&\quad\quad\quad\quad\quad\quad\quad\quad\quad\quad\quad\quad\quad\quad
\left.
+{3\zeta\log t \over \counter{i}{t-1}{\hidervectorm} } ,1\sizecorr{\meanw{i}{t-1} + \sqrt{2\zeta\meanw{i}{t-1}(1-\meanw{i}{t-1})\log t\over \counter{i}{t-1}{\hidervectorm}}}\right\}\CommaBin
                       \end{align*}
\noindent
where we choose the exploration factor to be $\zeta\triangleq1.2$.  Notice that we could take any $\zeta>1$ as shown by \citet{audibert2009}. We provide the policy $\pi_{\CUCBV}$ that we consider in Algorithm~\ref{CUCBV}. 
\setlength{\textfloatsep}{10pt}
\begin{algorithm}
\caption{Combinatorial upper confidence bounds with variance estimates (\CUCBV) for \SnS}\label{CUCBV}
\begin{algorithmic}
\STATE \textbf{Input}: $\DAG$.
\FOR{$t=1..\infty$}
\STATE select $\rbs_t$ given by $\Oracle{\vUCBw{t},\vUCBc{t},\DAG}$.
\STATE perform $\rbs_t[{\rhidervector_t}]$.
\STATE collect feedback and update counters and empirical averages according to~\eqref{counter} and~\eqref{average}.
\ENDFOR
\end{algorithmic}
\end{algorithm}
\subsection{Analysis}
Notice that since an arm $i\in \rbs_t$ is pulled (and thus $\rcost_{i,t}$ is revealed to the agent) with probability $1-\be_{\bs[i-1]}\transpose\hidervector$ over round $t$, we fall into the setting of \emph{probabilistically triggered arms}\/ w.r.t.\,costs, described by \citet{chen2015combinatorial} and \citet{wang2017improving}. Thus we could rely on these prior results. However, the main difficulty in our setting is that we also need to deal with probabilities $\rhider_{i,t}$, that the agent actually observes for every arm $i$ in~$\rbs_t$, either because it actually pulls arm $i$, or because it deduces the value from other pulls of round~$t$. In particular, if we follow the analysis of \citet{chen2015combinatorial} and \citet{wang2017improving}, the double sum in the definition of $J$ leads to  
 expected regret bound that is quite large. Indeed, assuming that all costs are deterministically equal to $1$, if we suffer an error of $\delta$ when approximating each $\hider_i$, then the global error can be as large as $\sum_{i=1}^\arms\sum_{j=1}^{i-1}\delta=\cO(\arms^2\delta)$, contrary to just  $\cO(\arms\delta)$ for the approximation error w.r.t.\,costs, that is more common in combinatorial semi-bandits. Thus,  we rather combine their work with the variance estimates of $\hider_{i}$. Often, this does not provide a significant improvement over \UCB in terms of expected regret (otherwise we could do the same for the costs), but since in our case, the variance is of order ${1/ \arms}$, the gain is non-negligible.\footnote{The error $\delta$ is thus scaled by the standard deviation, of order $1/\sqrt{n}$, giving a global error of $\cO(\arms^{1.5}\delta)$. We therefore recover the factor $\arms^{1.5}$ given in Theorem~\ref{banditregret}.}  
We let $\costm_{\min}>0$ be any deterministic lower bound on the set $\sset{\be_{\bs_u[\rhidervector_u]}\transpose \costvector, u\geq 1}$. Furthermore, we let
\[T_B\triangleq\ceil{2B/\costm_{\min}}\] and for any search $\mathbf{s}$, we define the \emph{gap} of $\bs$ as
\begin{align*}\gapss{\mathbf{s}}&\triangleq \be_{\bs}\transpose\hidervector\left({\J(\mathbf{s})\over \J^\star}-1\right)\\
&=\frac{1}{\J^\star}\sum_{i=1}^{\abs{\bs}}\cost_{s_i}\pa{1-\sum_{j=1}^{i-1}\hider_{s_j}}-\sum_{i=1}^{\abs{\bs}}\hider_{s_i}\geq 0,\end{align*}
that represents the \emph{local regret} of selecting a sub-optimal search $\mathbf{s}$ at some round. In addition, for each arm $i \in [\arms]$, we define \[\Delta_{i,\min}\triangleq\inf_{\mathbf{s}\notin \search^\star:~i \in \mathbf{s}}\Delta\pa{\mathbf{s}}>0.\]
We 
provide bounds for the expected regret of $\pi_{\CUCBV}$ in Theorem~\ref{banditregret}. 
%
The first bound is $\distribution$-dependent, and is characterized
by $\costm_{\min}$, $\J^\star$, and $\sigma_i^2$ , $\Delta_{i,\min}, ~i\in [\arms]$. Its main term scales logarithmically w.r.t.\,$B$. The second bound is true for any \SnS problem instance having a fixed value of $\costm_{\min}>0$ and $\J^\star>0$.

\begin{restatable}{theorem}{restaE}\label{banditregret}
The expected regret of {\CUCBV} satisfies
\begin{align*}&\RB_B(\pi_{\CUCBV})=\\
 &\cO\!\pa{\!{\arms\log T_B}\!\!\sum_{i\in [\arms]}\!\!{ 
\frac{1\!+\!\pa{\J^\star\!\!+\!\arms}^2\!\sigma_i^2}{ {\J^\star}^2\Delta_{i,\min} }\! +\!\frac{\pa{\J^\star\!\!+\!\arms}}{\J^\star}\!\log\!\pa{\!\frac{\arms}{\J^\star\!\Delta_{i,\min} }\!}}\!\!}.
\end{align*}
In addition, 
\begin{align*}\sup
 \RB_B(\pi_{\CUCBV})=\cO\pa{\!
 \sqrt{\arms}\left(1\!+\!{\arms \over {\J^\star}}\right)
\sqrt{T_B\log T_B}},\end{align*}
where the $\sup$ is taken over all possible \SnS problems with fixed $\costm_{\min}$ and $\J^\star$. 
\end{restatable}
\noindent
The proof is in Appendix~\ref{app:banditregret}. Recall the main challenge comes from the estimation of $\hidervector$ and not from $\costvector$. Our analysis uses \emph{triggering probability groups}\/ and the \emph{reverse amortization trick} of \cite{wang2017improving} for dealing with costs. However, for hider probabilities, only the second trick is necessary.\footnote{When we select search $\bs$, all feedback $\rhider_i,~ i\in \bs$ is received with probability 1, so  \emph{triggering probability groups}
 are not useful.} We use it not only to deal with the slowly concentrating confidence term for the estimates of each arm~$i$, 
but also to completely amortize the additional fast-rate confidence term due to variance estimation
coming from the use of Bernstein's inequality.
However,  the analysis of \cite{wang2017improving} only considers a deterministic horizon. In our case, we need to deal with a \emph{random-time} horizon. For that, notice that their regret upper bounds that hold in expectation are obtained by splitting the expectation into two parts. The 
 first part is filtered with a high-probability event on which the regret grows as the logarithm of the random horizon and the second one is filtered with a low-probability event, on which we bound the regret  by a constant.
Since the $\log$ function is concave, we can upper bound the expected regret by a term growing as the logarithm of the expectation of the random horizon, with Jensen's inequality. Finally, we upper bound  the expectation of the random horizon to get the rate of $\log T_B$.

\subsection{Tightness of our regret bounds}
Since we succeeded in reducing the dependence on $\arms$ in the expected regret with 
confidence bounds based on variance estimates, we can now ask whether this dependence in Theorem~\ref{banditregret}
is tight.
We stress that our solution to \SnS is \textit{computationally 
efficient}. In particular, \emph{both} the offline oracle optimization and the computation of the optimistic search $\bs_t$ in the online part \emph{are tractable}.

Whenever rewards are not arbitrary correlated (as is the case in our setting), we can potentially exploit these correlations in order to reduce the regret's dependence on $\arms$ even further. This could be done by choosing a tighter confidence region 
such as a confidence ellipsoid (\citealp{DBLP:journals/corr/DegenneP16}), or a KL-confidence ball (\citealp{combes2015combinatorial})  instead of coordinate-wise confidence intervals.
Unfortunately, these do not lead to computationally efficient algorithms. Notice that given an \emph{infinite} computational power, our dependence on $\arms$ is not tight. In particular, there is an extra $\sqrt{\arms}$  factor in our gap-free bound (see Theorem~\ref{lower}). 
It is an open question whether a better \textit{efficient} policy exists. 

To show that we are only a  $\sqrt{\arms}$ factor away, in the following theorem we provide a class of \SnS problems  (parameterized by $\arms$ and~$B$) on which the regret bound provided in Theorem~\ref{banditregret} is tight up to a $\sqrt{\arms}$ factor (and a logarithmic one).

\begin{restatable}{theorem}{restaL}\label{lower}
For simplicity, let us assume that $\arms$ is even and that $B$ is a multiple of $\arms$. For any optimal online policy $\pi$, there is a \SnS problem with $\arms$ arms and budget $B$ such that
\[-4+\frac{1}{28}\sqrt{\frac{B}{\arms}}\leq\RB_B({\pi})=\cO\pa{\sqrt{B\log\pa{\frac{B}{n}}}}.\]\end{restatable}
For the proof, we consider a DAG 
 composed of two disjoint paths (Figure~\ref{2paths}), with all costs deterministically set to $1$ and with the hider located either at $a_{\arms/2}$ or $b_{\arms/2}$. This information is given to the agent. We then reduced this setting to a two-arm bandit over at least $B/\arms$ rounds. The complete proof is in Appendix~\ref{app:lower}. 
\begin{restatable}{ }{restaM}\label{paths}
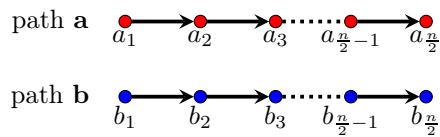
\begin{figure}[H]
\centering
\begin{tikzpicture}
  \node[draw, shape=circle, fill=blue, scale=0.5] (m1) at (0,0) {};
   \node[ below] at (m1) {$b_1$};
  \node[draw, shape=circle, fill=blue, scale=0.5] (m2) at (1,0) {};
  \node[ below] at (m2) {$b_2$};
  \node[draw, shape=circle, fill=blue, scale=0.5] (m3) at (2,0) {};
  \node[ below] at (m3) {$b_3$};
  \node[draw, shape=circle, fill=blue, scale=0.5] (m4) at (3,0) {};
  \node[ below] at (m4) {$b_{\frac{\arms}{2}-1}$};
  \node[draw, shape=circle, fill=blue, scale=0.5] (m5) at (4,0) {};
  \node[ below] at (m5) {$b_{\frac{\arms}{2}}$};
  \node[draw, shape=circle, fill=red, scale=0.5] (n1) at (0,1) {};
    \node[ below] at (n1) {$a_1$};
  \node[draw, shape=circle, fill=red, scale=0.5] (n2) at (1,1) {};
  \node[ below] at (n2) {$a_2$};
  \node[draw, shape=circle, fill=red, scale=0.5] (n3) at (2,1) {};
  \node[ below] at (n3) {$a_3$};
  \node[draw, shape=circle, fill=red, scale=0.5] (n4) at (3,1) {};
  \node[ below] at (n4) {$a_{\frac{\arms}{2}-1}$};
  \node[draw, shape=circle, fill=red, scale=0.5] (n5) at (4,1) {};
  \node[ below] at (n5) {$a_{\frac{\arms}{2}}$};
  \node[   scale=1] (n6) at (-1,1) {path $\ba$};
    \node[   scale=1] (m6) at (-1,0) {path $\bb$};

    \draw [->, line width=.5mm,>=stealth, black] (n1) -- (n2);
  \draw [->, line width=.5mm,>=stealth, black] (n2) -- (n3);
  \draw [dotted, line width=.5mm,>=stealth, black] (n3) -- (n4);
  \draw [->, line width=.5mm,>=stealth, black] (n4) -- (n5);

  \draw [->, line width=.5mm,>=stealth, black] (m1) -- (m2);
  \draw [->, line width=.5mm,>=stealth, black] (m2) -- (m3);
  \draw [dotted, line width=.5mm,>=stealth, black] (m3) -- (m4);
  \draw [->, line width=.5mm,>=stealth, black] (m4) -- (m5);
\end{tikzpicture}
\caption{The DAG considered in Theorem~\ref{lower}. 
}
\label{2paths}
\end{figure}
\end{restatable}
Notice that bounds provided in Theorem~\ref{lower} \emph{decrease} with $\arms$. This is because, in the  \SnS problem, the increasing dependence on $\arms$ is counterbalanced by the fact that the number of rounds is of order $B/ \arms$, and that $\J^\star$ is of order $\arms$. 
\section{Experiment}
\label{s:ex}
In this section, we present an experiment for \SnS.
We compare our \CUCBV with three other online algorithms, which are same as \CUCBV except for the estimator $\hidervectorm_t$ to be plugged in $\oracle$. We give corresponding definitions of $\hidervectorm_t$ in Table~\ref{table:exp_algos}, where we take $\zeta\triangleq1.2$, and where \[\kl{p}{q}\triangleq p\log\pa{\frac{p}{q}} + (1-p)\log\pa{\frac{1-p}{1-q}}\] is
the Kullback-Leibler divergence between two Bernoulli distributions of parameters $p,q\in [0,1]$ respectively.
\begin{table}[t]
\vspace{-0.1in}
\caption{Comparison algorithms in the experiment.}
\vspace{0.1in}
\centering
\label{table:exp_algos}
\begin{tabular}{|c|c|}
\hline
Algorithm & Definition of $\hiderm_{i,t}$ \\
\hline
\CUCB & \tableq{\min\sset{ \meanw{i}{t-1} 
+\sqrt{0.5\zeta\log t \over \counter{i}{t-1}{\hidervectorm}}\CommaBin 1}}
\\
\hline
\CUCBKL & \tabtab{The unique solution $x$ to\\\tableq{\counter{i}{t-1}{\hidervectorm}\kl{\meanw{i}{t-1}}{x}=\zeta\log t}  \\ such that $x\in [\meanw{i}{t-1},1]$}
\\
\hline
\tabtab{\textsc{Thompson}\\\textsc{Sampling}} &
\tabtab{An independent sample from\\\tableq{\text{Beta}\pa{\alpha,\counter{i}{t-1}{\hidervectorm}-\alpha}},\\ where  \tableq{\alpha=\counter{i}{t-1}{\hidervectorm}\meanw{i}{t-1}}}
\\ \hline
\end{tabular}
\end{table}
We run simulations
for all the algorithms with $\arms=100$ and without precedence constraints, i.e., when the DAG is an edgeless graph. Notice that in this case, a search can be any ordered subset of arms (thus, the set of possible searches is of cardinality~$\sum_{k=0}^{\arms}{\arms !}/{k!}\leq e n!$). This restriction does not remove complexity from the online problem, but rather from the offline one, so even in that case, the online problem is challenging.
We take parameter $\hidervector$ defined as
\[
\left\{
    \begin{aligned}
        \hider_i&=\frac{1}{2^i}\quad \text{for }i\in [m-1] \\
        \hider_{m}&=\pa{\frac{1}{2}+\varepsilon}\hider_{m-1}\\
        \hider_i&=\pa{\frac{1}{2}-\varepsilon}\frac{\hider_{m-1}}{\arms-m}\quad \text{for }i\in \sset{m+1,\dots,\arms},
    \end{aligned}
\right.\]
 where we chose $m\triangleq40$. 
 For $\varepsilon\in (0,1/2)$, one can see that $\search^\star=\sset{[m]}$. Intuitively, 
 $\hider_i$ models the proportion of users answering $i$ to some fixed request:\footnote{For recommender systems or search engines, $\hider_i$ can thus be seen as the probability that an user aims to find $i$ when entering the request.} When $\varepsilon=0$,  half of the population answers $1$, a quarter answers~$2$, \dots, until $m$, and remaining users answer uniformly on remaining arms $\sset{m+1,\dots,\arms}.$  
We chose $\varepsilon=0.1$, $\cost_i=1/2$ for all $i\in [\arms]$ and take $\rcost_i\sim\Bernoulli(\cost_i)$. 
In Figure~\ref{exp:logscale}, for each algorithm considered, we plot the quantity 
\[\frac{B}{\J^\star}-\sum_{t=1}^{\rd{\tau}_B-1} \be_{\rbs_t[\rhidervector_t]}\transpose \rhidervector_t,\] 
with respect to budget $B$, averaged over $100$ simulations.
As shown in Proposition~\ref{prop:approx}, the  curves obtained this way provide good approximations
to the true regret
curves. 
We notice that \CUCBKL, \CUCBV, and \ThompsonSampling are significantly better than \CUCB, since the latter explores too much. In addition, the regret curves of \CUCBKL, \CUCBV and \ThompsonSampling are quite similar. In particular, their asymptotic slopes seem equal, which hints that regret rates are comparable on this instance.

%
  
\begin{figure}[t]
\vspace{-0.5cm}
\centering
\resizebox{\columnwidth}{!}{\input{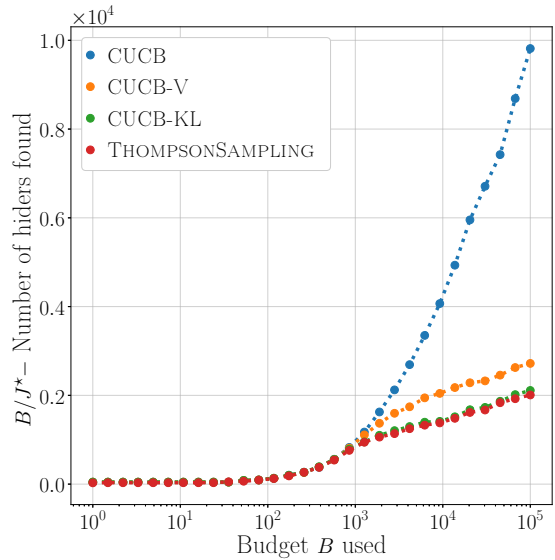}}
\vspace{-0.5cm}
\caption{Cumulative regret for \SnS, with $B$ up to $10^5$, averaged over $100$ independent simulations.}
\label{exp:logscale}
\end{figure}
  
\section{Conclusion and future work}
We presented \SnS problem and provided a stationary offline solution. We gave theoretical guarantees on its optimality and proved that it is computationally efficient. 
We also considered the learning extension of the problem
where the distribution of the hider and the cost are not known. We provided \CUCBV, an upper-confidence bound approach,
tailored to our case and gave expected regret guarantees with respect to the optimal policy.

We now discuss several possible extensions of our work. We could consider several hiders rather than just one.
Another would be to explore the Thomson sampling \citep{chapelle2011empirical,agrawal2011analysis,Komiyama2015,Wang2018}  further in the
learning case by considering a Dirichlet prior on the \emph{whole} arm set. The Dirichlet seems 
appropriate  
because a sample~$\hidervectorm$ from this prior is in the simplex. 
The main drawback however is the difficulty of \emph{efficiently} updating such prior 
to get the posterior, because in the case when the hider is not found, the one-hot vector is not received entirely.
\paragraph{Acknowledgements}
Vianney Perchet has benefited from the support of the ANR (grant n.ANR-13-JS01-0004-01), of the FMJH Program Gaspard Monge in optimization and operations research (supported in part by EDF), from the Labex LMH and from the CNRS through the PEPS program. The research presented was also supported by European CHIST-ERA project DELTA, French Ministry of
Higher Education and Research, Nord-Pas-de-Calais Regional Council,
Inria and Otto-von-Guericke-Universit\"at Magdeburg associated-team north-european project Allocate, and French National Research Agency projects ExTra-Learn (grant n.ANR-14-CE24-0010-01) and BoB (grant n.ANR-16-CE23-0003), 
FMJH Program PGMO with the support to this program from Criteo.
\bibliography{library,library-all,example}

\appendix
\onecolumn
\section{Proof of Theorem~\ref{thm:oracle}}
\label{app:sidney}

\setcounter{scratchcounter}{\value{theorem}}\setcounter{theorem}{\the\numexpr\getrefnumber{thm:oracle}-1}\restaA*\setcounter{theorem}{\the\numexpr\value{scratchcounter}}

Here, we might abbreviate $\obj{\J}{~\!\!\cdot~\!\!}{\hidervectorm,\costvectorm}^+$ into $\Jp$, and $\obj{\J}{~\!\!\cdot~\!\!}{\hidervectorm,\costvectorm}$ into $\J$, keeping in mind that our results will be valid for all $\pa{\hidervectorm,\costvectorm}\in \pa{\R_+^\arms}^2$.
To prove Theorem~\ref{thm:oracle} we first define the concept of \textsl{density}, well know in scheduling and search theory.
\begin{definition}[Density] The density is the function defined on $A\in\mathcal{P}([\arms])$ by $\rho(A)\triangleq{\be_A\transpose\hidervectorm/ \be_A\transpose\costvectorm }$, and $\rho\pa{\emptyset}=0$.
\end{definition}
\noindent
Density of $A \subset [\arms]$ can be understood as the quality/price ratio of that set of arms: the quality is the overall probability of finding the hider in it, while the price is the total cost to fully explore it. Without precedence constraint, the so-called
Smith's rule of ratio \citep{smith1956various} gives that $\bs$ minimizes $\obj{\D}{~\!\!\cdot~\!\!}{\hidervectorm,\costvectorm}$ over linear orders (i.e., permutations of $[\arms]$) if and only if\footnote{One can see that $\sum_{\{i,j\}\in I(\sigma),~i<j}\costm_{s_i}\costm_{s_j}\left(\rho(s_i)-\rho(s_j)\right)$ is the variation of $\D$ when swapping a linear order $\bs$ by a permutation $\sigma$, where $I(\sigma)$ the set of inversions in $\sigma$.} $\rho(s_1)\geq \dots \geq \rho(s_\arms)$.
\citet{Sidney1975} generalized this principle to any precedence constraint with the concept of Sidney decomposition. Recall that an initial set is the support of a search. 
\begin{definition}[Sidney decomposition]
 A Sidney decomposition  $\pa{A_1,A_2,\dots,A_k}$ is an ordered partition of $[\arms]$ such that for all $i\in[k]$, $A_i$ is an  initial set of maximum density in $\DAG\shrink{A_i\sqcup\dots\sqcup A_{k}}$.
\end{definition}

\noindent Notice that the Sidney decomposition defines a more refined poset on $[\arms]$, with the extra constraint that an element of $A_i$ must be processed before those of $A_j$ for $i <j$. 
Any $\DAG$-linear extension that is also a linear extension of this poset is said to be \emph{consistent} with the Sidney decomposition.
The following theorem was proved by \citet{Sidney1975}:

\begin{theorem}[\citealp{Sidney1975}]
 Every minimizer of $\obj{\D}{~\!\!\cdot~\!\!}{\hidervectorm,\costvectorm}$ over $\DAG$-linear extensions is consistent with some Sidney decomposition. Moreover, for every Sidney decomposition $(A_1,\dots,A_k)$, there is a minimizer of $\obj{\D}{~\!\!\cdot~\!\!}{\hidervectorm,\costvectorm}$ over $\DAG$-linear extensions that is consistent with $(A_1,\dots,A_k)$.
 \label{thm:sidney}
\end{theorem}

\noindent
Notice that Theorem~\ref{thm:sidney} does not provide a full characterization of minimizers of $\obj{\D}{~\!\!\cdot~\!\!}{\hidervectorm,\costvectorm}$ over $\DAG$-linear extensions, but only a \emph{necessary} condition. Nothing is stated about how to chose the ordering inside each $A_i$'s, and this highly depends on the structure of $\DAG$ \citep{Lawler1978,Ambuhl2009,Ambuhl2011}. 
We are now ready to prove Theorem~\ref{thm:oracle}, thanks to Lemma~\ref{lem:subsidney}, of which the proof is given in Appendix~\ref{app:subsidney}.

\begin{restatable}{lemma}{probaB}
\label{lem:subsidney}
 For any Sidney decomposition $(A_1,\dots,A_k)$, there exists $i\leq k$ and a search with support $A_1\sqcup\dots\sqcup A_i$ that minimizes $\Jp$.
\end{restatable}

\begin{proof}[Proof of Theorem~\ref{thm:oracle}] We know
from first statement of Theorem~\ref{thm:sidney} that $\bs\triangleq\sche\pa{\hidervectorm,\costvectorm,\DAG}$ given in Algorithm~\ref{algo:oracle} is consistent with some Sidney decomposition $(A_1,\dots,A_k)$. Let $i\leq k$ and $\ba$ minimizing $\Jp$ of support $A_1\sqcup\dots\sqcup A_i$ given by Lemma~\ref{lem:subsidney}. Let $\bs=\bs_1\bs_2$ with $\bs_1$ being the restriction of $\bs$ to $A_1\sqcup\dots\sqcup A_i$ (and thus $\bs_2$ is its restriction to $A_{i+1}\sqcup\dots\sqcup A_k$ ). Let's prove that $\bs_1$ is also a minimizer of $\Jp$ by showing $\Jp\pa{\bs_1 }\leq \Jp\pa{\ba }$, thereby concluding the proof.
Since $0\leq\obj{\D}{\ba\bs_2}{\hidervectorm,\costvectorm}- \obj{\D}{\bs_1\bs_2}{\hidervectorm,\costvectorm}=\obj{\D}{\ba}{\hidervectorm,\costvectorm}- \obj{\D}{\bs_1}{\hidervectorm,\costvectorm}$, we have 

\[{\obj{\D}{\bs_{1} }{\hidervectorm,\costvectorm}+(1-\hidervectorm\transpose\be_{A_1\sqcup\dots\sqcup A_i})\costvectorm\transpose\be_{A_1\sqcup\dots\sqcup A_i}\over \hidervectorm\transpose\be_{A_1\sqcup\dots\sqcup A_i} }\leq {\obj{\D}{\ba }{\hidervectorm,\costvectorm}+(1-\hidervectorm\transpose\be_{A_1\sqcup\dots\sqcup A_i})\costvectorm\transpose\be_{A_1\sqcup\dots\sqcup A_i}\over \hidervectorm\transpose\be_{A_1\sqcup\dots\sqcup A_i} }\CommaBin\]
i.e., $\J\pa{\bs_1 }\leq\J\pa{\ba}$, and because $x\mapsto x^+$ is non-deacreasing on $\R$, we have $\Jp\pa{\bs_1 }\leq\Jp\pa{\ba}$.
\end{proof}

\noindent
The proof of Lemma~\ref{lem:subsidney} also uses Sidney's Theorem~\ref{thm:sidney}, but this time the second statement. However, although it
provides a crucial analysis, with fixed support, concerning the order to choose for minimizing $\obj{\D}{~\!\!\cdot~\!\!}{\hidervectorm,\costvectorm}$ and therefore $\obj{\J}{~\!\!\cdot~\!\!}{\hidervectorm,\costvectorm}^+$, nothing is said about the support to choose. Thus, to prove Lemma~\ref{lem:subsidney}, we also need the following Proposition~\ref{prop:support}, that gives the key support property satisfied by $\Jp$.

\begin{proposition}[Support property]
If $\mathbf{xy},\mathbf{xyz}\in \search$ with $\rho(\mathbf{z})\geq \rho(\mathbf{y})$, then
\begin{align}\label{eq:sp}
\Jp(\mathbf{xy})\geq \min\sset{\Jp(\mathbf{x}),\Jp(\mathbf{xyz})}.
\end{align}
\label{prop:support}
\end{proposition}
\begin{proof}
If $\J({\mathbf{xyz}})<0$, then $\J^+({\mathbf{xyz}})=0\leq \J^+({\mathbf{xy}})$ and \eqref{eq:sp} is true. We thus assume $\J({\mathbf{xyz}})\geq 0$.
Since $\J({\mathbf z})\leq \frac{1}{\rho({\mathbf z})}\CommaBin$ 
\begin{equation}0\leq \J({\mathbf{xyz}})={\J({\mathbf{xy}})\hidervectorm\transpose\be_{\mathbf{xy}}\over \hidervectorm\transpose\be_{\mathbf{xyz}}}
+ {\hidervectorm\transpose\be_{\mathbf z}\J(\mathbf{z})-\hidervectorm\transpose\be_{\mathbf{xy}}\costvectorm\transpose\be_{\mathbf z}\over\hidervectorm\transpose\be_{\mathbf{xyz}}}\leq {\J({\mathbf{xy}})\hidervectorm\transpose\be_{\mathbf{xy}}\over \hidervectorm\transpose\be_{\mathbf{xyz}}} + {\hidervectorm\transpose\be_{\mathbf z}(1-\hidervectorm\transpose\be_{\mathbf{xy}})\over\rho(\mathbf{z})\hidervectorm\transpose\be_{\mathbf{xyz}}}\cdot
\label{proof_support0}
\end{equation}
If $1-\hidervectorm\transpose\be_{\mathbf{xy}}\leq 0$, by \eqref{proof_support0}, we have that $$0\leq \J({\mathbf{xyz}})\leq {\J({\mathbf{xy}})\hidervectorm\transpose\be_{\mathbf{xy}}\over \hidervectorm\transpose\be_{\mathbf{xyz}}}\leq\J({\mathbf{xy}}),$$
so $\Jp({\mathbf{xyz}})\leq\Jp({\mathbf{xy}})$ and \eqref{eq:sp} is true. Thus, we
suppose that $1-\hidervectorm\transpose\be_{\mathbf{xy}}\geq 0$.
If $\J(\mathbf{x})\leq \J({\mathbf{xy}})$, then $\Jp(\mathbf{x})\leq \Jp({\mathbf{xy}})$ and \eqref{eq:sp} is true. Else, \begin{equation}\label{proof_support1}\J({\mathbf{xy}})\geq {1\over \hidervectorm\transpose\be_{\mathbf y}}\pa{\J({\mathbf{xy}})\hidervectorm\transpose\be_{\mathbf{xy}}-\J(\mathbf{x})\hidervectorm\transpose\be_{\mathbf x}}=\sum_{i=1}^{\abs{\mathbf{y}}}\frac{\costm_{y_i}\pa{1-\hidervectorm\transpose\pa{\be_\mathbf{x}+\be_{\by[i-1]}}}}{\hiderm_\mathbf{y}}\geq {1-\hidervectorm\transpose\be_{\mathbf{xy}}\over \rho\pa{\mathbf y}}\cdot\end{equation}
Thus, we have
\begin{align*}\J({\mathbf{xyz}})-\J({\mathbf{xy}})\leq&
{\J({\mathbf{xy}})\hidervectorm\transpose\be_{\mathbf{xy}}\over \hidervectorm\transpose\be_{\mathbf{xyz}}} + {\hidervectorm\transpose\be_{\mathbf z}(1-\hidervectorm\transpose\be_{\mathbf{xy}})\over\rho(\mathbf{z})\hidervectorm\transpose\be_{\mathbf{xyz}}}-\J({\mathbf{xy}})&\text{using}~\eqref{proof_support0}. 
\\
=&
{-\hidervectorm\transpose\be_{\mathbf z}\J({\mathbf{xy}})\over\hidervectorm\transpose\be_{\mathbf{xyz}}}
+ {\hidervectorm\transpose\be_{\mathbf z}(1-\hidervectorm\transpose\be_{\mathbf{xy}})\over\rho(\mathbf{z})\hidervectorm\transpose\be_{\mathbf{xyz}}} 
\\
\leq&
\frac{\hidervectorm\transpose\be_{\mathbf z}}{\hidervectorm\transpose\be_{\mathbf{xyz}}}\pa{
{-({1-\hidervectorm\transpose\be_{\mathbf{xy}} )}\over\rho\pa{\mathbf y}}
+ {1-\hidervectorm\transpose\be_{\mathbf{xy}}\over\rho({\mathbf z})}}\leq 0 &\text{using}~\eqref{proof_support1},~\text{and then}~\rho(\mathbf{z})\geq \rho(\mathbf{y}). 
\end{align*}
So, $\Jp({\mathbf{xyz}})\leq\Jp({\mathbf{xy}})$ and \eqref{eq:sp} is true.
\end{proof}

\begin{example}
  Now, as a preview, we can actually derive easily the proof of Lemma~\ref{lem:subsidney} when there is no precedence constraints, the idea in the general case being very similar. Let $(A_1,\dots,A_k)$ be a Sidney decomposition. Then, if $a_{i,1},\dots,a_{i,j_i}$ are arms of $A_i$, we have
  \[\rho(a_{1,1})=\dots=\rho(a_{1,j_1})\geq \dots\geq\rho(a_{k,1})=\dots=\rho(a_{k,j_k}). \]
  Let $\bs^\star$ be a maximum-size minimizer of $\Jp$ of support $S$. Assume $S$ is not of the form given by Lemma~\ref{lem:subsidney}, and let~$x$ be the first, for the order $\pa{a_{1,1},\dots,a_{1,j_1},\dots,a_{k,1},\dots,a_{k,j_k} }$, in some $ A_i\backslash S$ while $S\cap \pa{A_i\sqcup\dots\sqcup A_k}\neq \emptyset$. By Proposition~\ref{prop:support}, we keep the optimality by either   adding $x$ to $\bs^\star$ (which contradicts the maximality of $\abs{\bs^\star}$), or by removing the suffix defined on $S\cap \pa{A_i\sqcup\dots\sqcup A_k}$, giving a support satisfying conclusion of Lemma~\ref{lem:subsidney}.  
\end{example}
\subsection{Proof of Lemma~\ref{lem:subsidney}}
\label{app:subsidney}

Before proving Lemma~\ref{lem:subsidney}, we state some preliminaries about initial sets of the DAG $\DAG$.

\begin{proposition}
$A$ is an initial set in $\DAG$ if and only if for all $a\in A$, the predecessors of $a$ in $\DAG$ are also in~$A$. 
\end{proposition}
\begin{proof}
 The direct sense is clear. Suppose now that for all $a\in A$, the predecessors of $a$ in $\DAG$ are also in $A$. Consider $\mathbf{a}=(a_1,\dots,a_{\abs{A}})$ a linear extension of $\DAG\shrink{A}$. Then it is a search, and predecessors of any $a_i$ in $\DAG$ are in $\sset{a_1,\dots,a_{i-1}}\cup A^c$, thus in $\sset{a_1,\dots,a_{i-1}}$ by assumption. Therefore, $\mathbf{a}$ is a search in $\DAG$ and $A$ is an initial set.
 \end{proof}
\noindent Let us recall that $\mathcal{L}\subset \mathcal{P}([\arms])$ is a lattice if $A,A'\in\mathcal{L}\Rightarrow \pa{A\cap A'\in\mathcal{L}~\text{and}~A\cup A'\in\mathcal{L}} $.

\begin{proposition}
The set of initial sets in $\DAG$ is a lattice. 
\end{proposition}
\begin{proof}
Let $A$ and $A'$ be two initial sets in $\DAG$. If $a\in A\cup A'$ (respectively $a\in A\cap A'$), then the predecessors of $a$ are included in predecessors of $A$ or (respectively and) the predecessors of $A'$, i.e., in $A$ or (respectively and) $A'$, so in $A\cup A'$ (respectively $A\cap A'$). 
\end{proof}
\noindent
Even if we do not use the following proposition,\footnote{Theorem~\ref{thm:sidney} does need this proposition.} we provide it nonetheless, since it illustrates  
how to handle density $\rho$.
\begin{proposition}
 The set of initial sets of maximum density in $\DAG$ is a lattice.
\end{proposition}

\begin{proof}
 We use the fact that for $a,b\geq 0$ and $a',b'>0$, $\frac{a+b}{a'+b'}\leq \max\sset{\frac{a}{a'},\frac{b}{b'}}$, with equality if and only if $\frac{a}{a'}=\frac{b}{b'}\cdot$ Indeed, if $A$ and $A'$ are two initial sets of maximum density in $\DAG$, then \[\frac{\hidervectorm\transpose\be_A}{\costvectorm\transpose\be_A}=\frac{\hidervectorm\transpose\pa{\be_A+\be_{A'}}}{\costvectorm\transpose\pa{\be_A+\be_{A'}}}=\frac{\hidervectorm\transpose\pa{\be_{A\cup A'}+\be_{A\cap A'}}}{\costvectorm\transpose\pa{\be_{A\cup A'}+\be_{A\cap A'}}}\leq \max\sset{\frac{\hidervectorm\transpose\be_{A\cup A'}}{\costvectorm\transpose\be_{A\cup A'}},\frac{\hidervectorm\transpose\be_{A\cap A'}}{\costvectorm\transpose\be_{A\cap A'}}}\!\cdot \]
 $A\cap A'$ and $A\cup A'$ are initial sets, so by maximality of density of $A$, $$ \max\sset{\frac{\hidervectorm\transpose\be_{A\cup A'}}{\costvectorm\transpose\be_{A\cup A'}},\frac{\hidervectorm\transpose\be_{A\cap A'}}{\costvectorm\transpose\be_{A\cap A'}}}\leq \frac{\hidervectorm\transpose\be_A}{\costvectorm\transpose\be_A}\cdot$$ Therefore, the equality holds, and it needs to be the case that $${\frac{\hidervectorm\transpose\be_{A\cup A'}}{\costvectorm\transpose\be_{A\cup A'}}=\frac{\hidervectorm\transpose\be_{A\cap A'}}{\costvectorm\transpose\be_{A\cap A'}}}=\frac{\hidervectorm\transpose\be_A}{\costvectorm\transpose\be_A},$$  so both $A\cap A'$ and $A\cup A'$ have maximum density. 
\end{proof}

\setcounter{scratchcounter}{\value{theorem}}\setcounter{theorem}{\the\numexpr\getrefnumber{lem:subsidney}-1}\probaB*\setcounter{theorem}{\the\numexpr\value{scratchcounter}}
\begin{proof}[Proof of Lemma~\ref{lem:subsidney}]
Let $j$ be the largest integer such that there is a search minimizing $\Jp$ of the form $\mathbf{a}_1\cdots \mathbf{a}_j\mathbf{a}$ with $\mathbf{a}_i$ of support $A_i$ for all $i\in [j]$. 
Let $\bs=\mathbf{a}_1\cdots \mathbf{a}_j\mathbf{a}$ be such search, with $\abs{\mathbf{s}}$ being the smallest possible. Let $A$ be the support of $\ba$. By contradiction, assume $A\neq \emptyset$. By Theorem~\ref{thm:sidney}, there exists a minimizer of the form $\mathbf{a}_{j+1}\by$ of $\obj{\D}{\mathbf{a}_1\cdots \mathbf{a}_j~\cdot~}{\hidervectorm,\costvectorm}$ over $\DAG\shrink{A_{j+1}\sqcup A}$-linear extensions, with $\mathbf{a}_{j+1}$ of support $A_{j+1}$. 
$A_{j+1}\cap A$ is an initial set of $\DAG\shrink{A_{j+1}\sqcup\cdots\sqcup A_k}$, therefore \[\rho(A_{j+1}\cap A)\leq \rho(A_{j+1})=\rho\pa{(A_{j+1}\cap A)\sqcup(A_{j+1}\backslash A)}\leq \rho(A_{j+1}\backslash A),\] and thus $\rho(A)\leq \rho(A_{j+1})\leq \rho(A_{j+1}\backslash A)$. If we let $\mathbf{b}$ be a search of $\DAG\shrink{\pa{A_{j+1}\backslash A} \sqcup A_{j+2}\sqcup\cdots\sqcup A_k}$ with support $A_{j+1}\backslash A$, then by Proposition~\ref{prop:support}, associated with $\obj{\D}{\mathbf{a}_1\cdots \mathbf{a}_j\ba_{j+1}\by}{\hidervectorm,\costvectorm}\leq \obj{\D}{\mathbf{a}_1\cdots \mathbf{a}_j\ba\bb}{\hidervectorm,\costvectorm}$, we have that
\[\Jp(\mathbf{s})\geq \min\sset{\Jp(\mathbf{a}_1\cdots \mathbf{a}_j),\Jp(\mathbf{a}_1\cdots \mathbf{a}_j\mathbf{ab})}\geq \min\sset{\Jp(\mathbf{a}_1\cdots \mathbf{a}_j),\Jp(\mathbf{a}_1\cdots \mathbf{a}_j\mathbf{a}_{j+1}\mathbf{y})},\]
contradicting either the definition of $j$ or the minimality of $\abs{\mathbf{s}}$.
\end{proof}

\section{Proof of Proposition~\ref{prop:approx}}
\label{app:approx}
\setcounter{scratchcounter}{\value{theorem}}\setcounter{theorem}{\the\numexpr\getrefnumber{prop:approx}-1}\restaC*\setcounter{theorem}{\the\numexpr\value{scratchcounter}}
\begin{proof}
If we let $B^0=B$, then for any offline policy $\pi$, if we denote by $\bs_t$ the search selected by $\pi$ at round~$t$ (we saw that an optimal policy \emph{selects} at the begining of a round a search and then \emph{performs} it),
and if we let $\rd{B}_t=B-\sum_{u=1}^t\be_{\rbs_u[\rhidervector_u]}\transpose\rcostvector_t$ be the remaining budget at time $t$,
\begin{align}
 F_B(\pi)=\sum_{t=1}^\infty\EE{ \sum_{i\in \bs_t}\II{\rd{B}_t\geq 0,~\rhider_{i,t}=1}}
&\leq
\label{Bt} \sum_{t=1}^\infty\EE{ \sum_{i\in \bs_t}\II{\rd{B}_{t-1}\geq 0,~\rhider_{i,t}=1}}\\
&=
\label{condition} \sum_{t=1}^\infty\EE{ \sum_{i\in \bs_t}\II{\rd{B}_{t-1}\geq 0}\hider_i}\\
&=
\nonumber \sum_{t=1}^\infty\EE{\II{\rd{B}_{t-1}\geq 0}\hidervector\transpose\be_{\bs_t}}\\
&=
\nonumber \sum_{t=1}^\infty \EE{\II{\rd{B}_{t-1}\geq 0}{\D({\bs_t})+(1-\hidervector\transpose\be_{\bs_t})\costvector\transpose\be_{\bs_t}\over \J({\bs_t})}}\\
&\leq
\nonumber\sum_{t=1}^\infty \EE{\II{\rd{B}_{t-1}\geq 0}{\D({\bs_t})+(1-\hidervector\transpose\be_{\bs_t})\costvector\transpose\be_{\bs_t}\over \J^\star}}\\
&= 
\label{tauB} {1\over  \J^\star}\EE{\sum_{t=1}^{\rd{\tau}_B}(\D({\bs_t})+(1-\hidervector\transpose\be_{\bs_t})\costvector\transpose\be_{\bs_t})}\\
&= 
\nonumber {1\over  \J^\star}\EE{\sum_{t=1}^{\rd{\tau}_B}\costvector\transpose\be_{\bs_t[\rhidervector_t]}}\\
&= 
 \nonumber{1\over  \J^\star}\EE{\sum_{t=1}^{\rd{\tau}_B-1}\rcostvector_t\transpose\be_{\bs_t[\rhidervector_t]}+\rcostvector_{\rd{\tau}_B}\transpose\be_{\bs_{\rd{\tau}_B}[\rhidervector_{\rd{\tau}_B}]}}\\
&\leq 
\label{tauB-1}{B+\arms\over  \J^\star} \CommaBin
\end{align}
\noindent
where \eqref{Bt} uses
$\rd{B}_t\geq 0 \Rightarrow \rd{B}_{t-1}\geq 0$, \eqref{condition}
is obtained by conditioning on previously sampled arms, \eqref{tauB} uses the random round
$\rd{\tau}_B$ such that $\rd{B}_{\rd{\tau}_B -1}\geq 0$ and $\rd{B}_{\rd{\tau}_B}<0$, and  
\eqref{tauB-1} uses the definition of  $\rd{B}_{\rd{\tau}_B-1}$ and $\rcost_{i,t}\leq 1.$
Now, for the lower bound, we have that
\begin{align}
\label{Btbis} F_B(\pi^\star)&\geq\sum_{t=1}^\infty\EE{ \sum_{i\in \mathbf{s}^\star}\II{\rd{B}_{t-1}\geq \arms,~\rhider_{i,t}=1}}\\
&=\sum_{t=1}^\infty\EE{ \II{\rd{B}_{t-1}\geq \arms}\be_{\mathbf{s}^\star}\transpose\hidervector}\label{same}\\
&={1\over \J^\star}\EE{ \sum_{t=1}^{\rd{\tau}} \rcostvector_{t}\transpose\be_{\bs^{*}[\rhidervector_t]}}\label{tB}\\
&\geq {B-\arms\over \J^\star}\label{deftB}\CommaBin
\end{align}
\noindent
where \eqref{Btbis} uses~$\rd{B}_{t-1}\geq \arms \Rightarrow \rd{B}_t\geq 0$,
\eqref{same} uses the same derivation as  previously, \eqref{tB} uses $\rd{\tau}$,
the random round such that $\rd{B}_{\rd{\tau}-1}\geq \arms \text{ and }\rd{B}_{\rd{\tau}}< \arms$, and \eqref{deftB}
is
$\text{by definition of }\rd{B}_{\rd{\tau}}$.
\end{proof}

\section{Proof of Theorem~\ref{banditregret}}
\label{app:banditregret}
\noindent
We let $\beta(t)\triangleq\inf_{1<\alpha\leq 3}\min\sset{\frac{\log t}{\log \alpha },t}t^{-\frac{\zeta}{\alpha}}$. In the proof of Theorem~\ref{banditregret}, we make several uses of the following concentration inequalities that use the same peeling  
argument for their proof as Theorem~1 of \cite{audibert2009} applied to original \emph{anytime} inequalities. 
\begin{fact}[Theorem~1 of~\citealp{audibert2009}]
Let $\pa{\rd{X}_t}$ be iid centered random variables 
with common support $[0, 1]$, $\rd{\bar{x}}_t\triangleq{1\over t}(\rd{X}_1+ \dots + \rd{X}_t)$ and let $\rd{v}_t\triangleq\frac{1}{t}\sum_{u=1}^t \pa{\rd{\bar{x}}_t-\rd{X}_u}^2$, then $$\PP{\exists u\leq t,~ \rd{\bar{x}}_u> \sqrt{2\rd{v}_u\zeta\log t\over u} + \frac{3\zeta\log t}{u}}\leq 3\beta(t).$$
\label{thm:audibert}
\end{fact}

\begin{fact}[\citealp{hoeffding1963probability,azuma1967weighted}]
Let $\pa{\rd{X}_t}$ be a martingale difference sequence with common support $[0, 1]$, and let $\rd{\bar{x}}_t\triangleq{1\over t}(\rd{X}_1+ \dots + \rd{X}_t)$, then $$\PP{\exists u\leq t,~ \rd{\bar{x}}_u> \sqrt{\zeta\log t\over 2u}}\leq \beta(t).$$
\label{hoeffding}
\end{fact}

\begin{fact}[Bernstein inequality]
Let $\pa{\rd{X}_t}$ be a martingale difference sequence with common support $[0, 1]$, $\sigma^2\triangleq\mathbb{V}\pa{\rd{X}_t}$, and let $\rd{\bar{x}}_t\triangleq{1\over t}(\rd{X}_1+ \dots + \rd{X}_t)$, then $$\PP{\exists u\leq t,~ \rd{\bar{x}}_u> \sqrt{2\sigma^2\zeta\log t\over u} + \frac{\zeta\log t}{3u}}\leq \beta(t).$$
\label{bernstein}
\end{fact}
Before we dive into the proof of Theorem~\ref{banditregret}, 
we first state a lemma that gives a high-probability control on the error that is made when estimating~$\hider_i$.
\begin{lemma} For any $i\in [\arms],$ and $t\geq 1$, 
 $$\PP{\UCBw{i}{t}-\hider_i > \sqrt{8\zeta\sigma_i^2\log t \over\counter{i}{t-1}{\hidervectorm}} + 
 {13.3\zeta\log t \over\counter{i}{t-1}{\hidervectorm}}}\leq 2\beta(t).$$ 
 \label{lem:variance}
\end{lemma}
\begin{proof}
Let $i\in [\arms],$ and $t\geq 1$. We define\begin{align*}
\rd{r}&\triangleq{8\zeta\log t \over \counter{i}{t-1}{\hidervectorm} }+2\sqrt{\pa{\sqrt{7}\zeta\log t \over \counter{i}{t-1}{\hidervectorm} }^2+{2\zeta{\sigma_i^2 }\log t \over \counter{i}{t-1}{\hidervectorm}}}\CommaBin\\
\rd{\delta}&\triangleq\sqrt{8\zeta\sigma_i^2\log t \over\counter{i}{t-1}{\hidervectorm}} + 
 {13.3\zeta\log t \over\counter{i}{t-1}{\hidervectorm}}\CommaBin\quad \text{and}\\
  \epsilon\pa{u}&\triangleq\sqrt{\frac{8\sigma_i^2\zeta\log\pa{t}}{u}}+\frac{2\zeta\log t}{3u} \cdot
 \end{align*}
 We have that
 \begin{align*}
  &\PP{\UCBw{i}{t}-\hider_i> \rd{\delta}}\\&=\PP{ \min\sset{ \meanw{i}{t-1} + \sqrt{2\zeta\meanw{i}{t-1}(1-\meanw{i}{t-1})\log t \over \counter{i}{t-1}{\hidervectorm}} + 
{3\zeta\log t \over \counter{i}{t-1}{\hidervectorm} } \CommaBin 1} -\hider_i > \rd{\delta} } \\
  &\leq \PP{  \meanw{i}{t-1} + \sqrt{2\zeta\meanw{i}{t-1}(1-\meanw{i}{t-1})\log t \over \counter{i}{t-1}{\hidervectorm}} + 
{3\zeta\log t \over \counter{i}{t-1}{\hidervectorm} }  -\hider_i > \rd{\delta} } \\
& \leq \PP{ \meanw{i}{t-1} + \sqrt{2\zeta\pa{\sigma_i^2 + \rd{\delta}/2}\log t \over \counter{i}{t-1}{\hidervectorm}} + 
{3\zeta\log t \over \counter{i}{t-1}{\hidervectorm} }  -\hider_i > \rd{\delta} } \\&+ \PP{\meanw{i}{t-1}(1-\meanw{i}{t-1}) > \sigma_i^2 + \rd{\delta}/2 }\!
 .\end{align*}
The first term is bounded by $\PP{\meanw{i}{t-1}-\hider_i> \rd{\delta}/2}$, as a consequence of \[\sqrt{2\zeta\pa{\sigma_i^2 + \rd{\delta}/2}\log t \over \counter{i}{t-1}{\hidervectorm}} + 
{3\zeta\log t \over \counter{i}{t-1}{\hidervectorm} }\leq \rd{\delta}/2.\]
Indeed, this holds if $\rd{\delta}$ is greater than the greatest root of the following second-degree polynomial of variable $x$: \[x^2/4-{4\zeta\log t \over \counter{i}{t-1}{\hidervectorm} }x+\pa{3\zeta\log t \over \counter{i}{t-1}{\hidervectorm} }^2 - {2\zeta{\sigma_i^2 }\log t \over \counter{i}{t-1}{\hidervectorm}}\cdot\]
But this root is $\rd{r}$, which is upper bounded by $\delta$ using the subadditivity of the square root.

For the second term, since $\meanw{i}{t-1}(1-\meanw{i}{t-1})=\meanw{i}{t-1}-2\hider_i\meanw{i}{t-1} + {\hider_i}^2 -\pa{\hider_i-\meanw{i}{t-1}}^2\!\!,$ \[\PP{\meanw{i}{t-1}(1-\meanw{i}{t-1}) \geq \sigma_i^2 + \rd{\delta}/2 }\leq \PP{\meanw{i}{t-1}-2\hider_i\meanw{i}{t-1} + {\hider_i}^2  \geq \sigma_i^2 + \rd{\delta}/2 }\!.\]
Hence, $\PP{\UCBw{i}{t}-\hider_i> \rd{\delta}}$ is bounded by \begin{align}\nonumber
  & \PP{\meanw{i}{t-1}-\hider_i> \rd{\delta}/2} + \PP{\meanw{i}{t-1}-2\hider_i\meanw{i}{t-1} + {\hider_i}^2  > \sigma_i^2 + \rd{\delta}/2 }
  \\ \label{eq:eps<del}& \leq  \PP{\meanw{i}{t-1}-\hider_i>\frac{\epsilon\pa{\counter{i}{t-1}{\hidervectorm}}}{2} } + \PP{\meanw{i}{t-1}-2\hider_i\meanw{i}{t-1} + {\hider_i}^2  > \sigma_i^2 + \frac{\epsilon\pa{\counter{i}{t-1}{\hidervectorm}}}{2} } \\ &\leq 
  \PP{\exists u\leq t,~\frac{1}{u}\sum_{v=1}^u\rhider_{i,v}-\hider_i>\frac{\epsilon\pa{u}}{2} } + \PP{\exists u\leq t,~\frac{1}{u}\sum_{v=1}^u\pa{\rhider_{i,v}-\hider_i}^2- \sigma_i^2   >  \frac{\epsilon\pa{u}}{2} } 
   \nonumber\\ &\leq 2\beta(t), \nonumber
  \end{align}
  where \eqref{eq:eps<del} uses ${\epsilon\pa{\counter{i}{t-1}{\hidervectorm}}}\leq \delta$ and the last inequality uses Bernstein’s inequality (Fact~\ref{bernstein}) twice, noticing that \[\mathbb{V}\pa{\frac{1}{u}\sum_{v=1}^u\pa{\rhider_{i,v}-\hider_i}^2}\leq \sigma_i^2.\]
\end{proof}
\setcounter{scratchcounter}{\value{theorem}}\setcounter{theorem}{\the\numexpr\getrefnumber{banditregret}-1}\restaE*\setcounter{theorem}{\the\numexpr\value{scratchcounter}}
\begin{proof}[Proof of  Theorem~\ref{banditregret}] 

We start with showing a lower bound on the expected reward of any policy $\pi$, \begin{align}
\label{as}F_B(\pi)&\geq \sum_{t\geq 1}\EE{\II{\rd{B}_{t-1}\geq \arms}\be_{\rbs_t}\transpose\hidervector}
\\
\nonumber &= \sum_{t\geq 1}\EE{\II{\rd{B}_{t-1}\geq \arms,~\rbs_t\in \search^\star}\be_{\rbs_t}\transpose\hidervector}
+
\sum_{t\geq 1}\EE{\II{\rd{B}_{t-1}\geq \arms,~\rbs_t\notin \search^\star}\be_{\rbs_t}\transpose\hidervector}
\\
\nonumber&= {1\over \J^\star}\sum_{t\geq 1}\EE{\II{\rd{B}_{t-1}\geq \arms,~\rbs_t\in \search^\star}\left({\D(\rbs_t)}+(1-\be_{\rbs_t}\transpose\hidervector)\be_{\rbs_t}\transpose\costvector\right)}
\\&\quad+
\nonumber{1\over \J^\star}\sum_{t\geq 1}\EE{\II{\rd{B}_{t-1}\geq \arms,~\rbs_t\notin \search^\star}{\left({\D(\rbs_t)}+(1-\be_{\rbs_t}\transpose\hidervector)\be_{\rbs_t}\transpose\costvector\right)}}
\\&\quad-
\nonumber\sum_{t\geq 1}\EE{\II{\rd{B}_{t-1}\geq \arms,~\rbs_t\notin \search^\star}\Delta\pa{\rbs_t}}
\\&=\nonumber
{1\over \J^\star}\sum_{t\geq 1}\EE{\II{\rd{B}_{t-1}\geq \arms}{{\left({\D(\rbs_t)}+(1-\be_{\rbs_t}\transpose\hidervector)\be_{\rbs_t}\transpose\costvector\right)}}}
\\\nonumber&\quad-
\sum_{t\geq 1}\EE{\II{\rd{B}_{t-1}\geq \arms,~\rbs_t\notin \search^\star}\Delta\pa{\rbs_t}}
\\&\geq {B-\arms \over \J^\star} -
\label{17}\sum_{t\geq 1}\EE{\II{\rd{B}_{t-1}\geq \arms,~\rbs_t\notin \search^\star}\Delta\pa{\rbs_t}}\!,
\end{align}
with \eqref{as} obtained as \eqref{Btbis} and \eqref{same}, and \eqref{17} as  \eqref{deftB}.
Therefore, since $F^\star_B\leq (B+\arms)/\J^\star$ 
by Proposition~\ref{prop:approx}
, we have that
\begin{align}
R_B(\pi)-{2\arms\over \J^\star}&\leq 
 \sum_{t\geq 1}\EE{\II{\rd{B}_{t-1}\geq \arms,~\rbs_t\notin \search^\star}\Delta\pa{\rbs_t}} \leq \sum_{t\geq 1}\EE{\II{\rd{B}_{t}\geq 0}\Delta\pa{\rbs_t}}=\EE{\sum_{t=1}^{\rd \tau_B -1}\Delta\pa{\rbs_t}}\!\!.\nonumber
\end{align}
\subsection{Bound on $\gapss{\rbs_t}$ under high probability events}
Since
$\rbs_t$ minimizes $\obj{\J}{~\!\!\cdot~\!\!}{\vUCBw{t},\vUCBc{t}}^+$, then $\obj{\J}{\rbs_t}{\vUCBw{t},\vUCBc{t}}^+\leq \obj{\J}{\bs^\star}{\vUCBw{t},\vUCBc{t}}^+$. In the following equations, small changes between successive lines are highlighted in red.
 \begin{align*}
  \gapss{\rbs_t}&={1\over \J^\star}\left(\sum_{i=1}^{|\rbs_t|}\cost_{\rs_{i,t}}\left(1-\hidervector\transpose\be_{\rbs_t[i-1]}\right) - \J^\star\be_{\rbs_t}\transpose\hidervector\right)\\
  &={1\over \J^\star}\left(\sum_{i=1}^{|\rbs_t|}\cost_{\rs_{i,t}}\left(1-\hidervector\transpose\be_{\rbs_t[i-1]}\right) - \textcolor{red}{\J^\star\vUCBw{t}\transpose\be_{\rbs_t}}\right)
  +\textcolor{red}{\pa{\vUCBw{t}-\hidervector}\transpose\be_{\rbs_t}}
  \\
    &={1\over \J^\star}\left(\sum_{i=1}^{|\rbs_t|}\cost_{\rs_{i,t}}\left(1-\hidervector\transpose\be_{\rbs_t[i-1]}\right) - \textcolor{red}{\obj{\J}{\bs^\star}{\vUCBw{t},\vUCBc{t}}^+}\vUCBw{t}\transpose\be_{\rbs_t}\right)
    + \textcolor{red}{{\obj{\J}{\bs^\star}{\vUCBw{t},\vUCBc{t}}^+-\J^\star\over \J^\star}\vUCBw{t}\transpose\be_{\rbs_t}}
  +\pa{\vUCBw{t}-\hidervector}\transpose\be_{\rbs_t}
  \\
&\leq{1\over \J^\star}\left(\sum_{i=1}^{|\rbs_t|}\cost_{\rs_{i,t}}\left(1-\hidervector\transpose\be_{\rbs_t[i-1]}\right) - \textcolor{red}{\obj{\J}{\rbs_t}{\vUCBw{t},\vUCBc{t}}^+}\vUCBw{t}\transpose\be_{\rbs_t}\right)
    + {\obj{\J}{\bs^\star}{\vUCBw{t},\vUCBc{t}}^+-\J^\star\over \J^\star}\vUCBw{t}\transpose\be_{\rbs_t}
  +\pa{\vUCBw{t}-\hidervector}\transpose\be_{\rbs_t}
    \\
    &\leq{1\over \J^\star}\left(\sum_{i=1}^{|\rbs_t|}\cost_{\rs_{i,t}}\left(1-\hidervector\transpose\be_{\rbs_t[i-1]}\right) -\textcolor{red}{ \obj{\J}{\rbs_t}{\vUCBw{t},\vUCBc{t}}}\vUCBw{t}\transpose\be_{\rbs_t}\right)
    + {\obj{\J}{\bs^\star}{\vUCBw{t},\vUCBc{t}}^+-\J^\star\over \J^\star}\vUCBw{t}\transpose\be_{\rbs_t}
  +\pa{\vUCBw{t}-\hidervector}\transpose\be_{\rbs_t}
    \\
&={1\over \J^\star}\sum_{i=1}^{|\rbs_t|}\left(\cost_{\rs_{i,t}}\left(1-\hidervector\transpose\be_{\rbs_t[i-1]}\right) - 
   \UCBc{\rs_{i,t}}{t}\left(1-\vUCBw{t}\transpose\be_{\rbs_t[i-1]}\right)            \right)
    + {\obj{\J}{\bs^\star}{\vUCBw{t},\vUCBc{t}}^+-\J^\star\over \J^\star}\vUCBw{t}\transpose\be_{\rbs_t}
  +\pa{\vUCBw{t}-\hidervector}\transpose\be_{\rbs_t}
      \\
&={1\over \J^\star}\left(\sum_{i=1}^{|\rbs_t|}\left(\cost_{\rs_{i,t}}-\UCBc{\rs_{i,t}}{t}\right)\left(1-\hidervector\transpose\be_{\rbs_t[i-1]}\right) + 
   \sum_{i=1}^{|\rbs_t|}\UCBc{\rs_{i,t}}{t}\left(\vUCBw{t}-\hidervector\right) \transpose  \be_{\rbs_t[i-1]}         \right)
    \\&\quad + {\obj{\J}{\bs^\star}{\vUCBw{t},\vUCBc{t}}^+-\J^\star\over \J^\star}\vUCBw{t}\transpose\be_{\rbs_t}
  +\pa{\vUCBw{t}-\hidervector}\transpose\be_{\rbs_t}\\
&\leq{1\over \J^\star}\left(\sum_{i=1}^{|\rbs_t|}\left(\cost_{\rs_{i,t}}-\UCBc{\rs_{i,t}}{t}\right)\left(1-\hidervector\transpose\be_{\rbs_t[i-1]}\right) + 
   \textcolor{red}{(\arms+\J^\star)}\pa{\vUCBw{t}-\hidervector}\transpose\textcolor{red}{\be_{\rbs_t}}            \right)
    + {\obj{\J}{\bs^\star}{\vUCBw{t},\vUCBc{t}}^+-\J^\star\over \J^\star}\vUCBw{t}\transpose\be_{\rbs_t}
   \\ &= \Delta_{\costvectorm}\!\pa{\rbs_t}+\Delta_{\hidervectorm}\!\pa{\rbs_t}+ {\obj{\J}{\bs^\star}{\vUCBw{t},\vUCBc{t}}^+-\J^\star\over \J^\star}\vUCBw{t}\transpose\be_{\rbs_t},
 \end{align*}
where
 $$\Delta_{\costvectorm}\!\pa{\rbs_t}\triangleq{1\over \J^\star}\sum_{i=1}^{|\rbs_t|}\left(\cost_{s_{i,t}}-
 \UCBc{\rs_{i,t}}{t}\right)\left(1-\hidervector\transpose\be_{\rbs_t[i-1]}\right)
 $$
 and 
  $$\Delta_{\hidervectorm}\!\pa{\rbs_t}\triangleq{\arms+\J^\star\over \J^\star}\pa{\vUCBw{t}-\hidervector}\transpose\be_{\rbs_t}\!.
 $$
\noindent
 For all $t\geq 1,$ we define the event
 \[\mathfrak{M}_t\triangleq\sset{\vUCBw{t}\geq \hidervector,\vUCBc{t}\leq \costvector }\!,\]
 under which ${\obj{\J}{\bs^\star}{\vUCBw{t},\vUCBc{t}}^+\leq\J^\star}$: indeed, we can first use $\vUCBc{t}\leq \costvector$ to write $\obj{\J}{\bs^\star}{\hidervector,\vUCBc{t}}\leq\obj{\J}{\bs^\star}{\hidervector,\costvector}=\J^\star$ because $\hidervector$ belongs to the simplex (thus, ${1-\be_{\bs[i-1]}\transpose\hidervector}\geq 0$ for all $i$). Then, using $\vUCBc{t}\geq 0$ with $\vUCBw{t}\geq \hidervector$, we can write $\obj{\J}{\bs^\star}{\vUCBw{t},\vUCBc{t}}\leq\obj{\J}{\bs^\star}{\hidervector,\vUCBc{t}}$. The result follows since $x\mapsto x^+$ is non-deacreasing on $\R$. Therefore, under $\mathfrak{M}_t$, 
 \[\gapss{\rbs_t}\leq \Delta_{\costvectorm}\!\pa{\rbs_t}+\Delta_{\hidervectorm}\!\pa{\rbs_t}. \]
 We define 
  \[\mathfrak{A}_t\triangleq\sset{\forall i\in \rbs_t, {13.3\zeta\log t \over\counter{i}{t-1}{\hidervectorm}} \leq {J^\star\gapss{\rbs_t}\over 2\arms(\arms+\J^\star)}}\!\CommaBin\]
 $$\mathfrak{N}_t\triangleq\left\{\forall i \in \rbs_t,~ 
 {\cost_i-\UCBc{i}{t}}\leq \sqrt{2\zeta\log t \over \counter{i}{t-1}{\costvectorm}}\text{ and }\UCBw{i}{t}-\hider_i\leq \sqrt{8\zeta\sigma_i^2\log t \over\counter{i}{t-1}{\hidervectorm}} + 
 {13.3\zeta\log t \over\counter{i}{t-1}{\hidervectorm}}\right\}\cdot$$
Under events $\mathfrak{A}_t,\mathfrak{M}_t,\mathfrak{N}_t $, we can write
\begin{align}\nonumber\gapss{\rbs_t} &\leq  \Delta_{\hidervectorm}\!\pa{\rbs_t} +  \Delta_{\costvectorm}\!\pa{\rbs_t}
 \\\nonumber
 &\leq -\gapss{\rbs_t} + 2\Delta_{\hidervectorm}\!\pa{\rbs_t} +  2\Delta_{\costvectorm}\!\pa{\rbs_t}
 \\\nonumber
 &=   {2\over \J^\star}\sum_{i\in \rbs_t}(\arms+\J^\star)\cdot\left(\UCBw{i}{t}-\hider_{i}-{J^\star\gapss{\rbs_t}\over 2\abs{\rbs_t}(\arms+\J^\star)}\right)+2\Delta_{\costvectorm}\!\pa{\rbs_t}\\\label{eq:eventn}
 &\leq {2\over \J^\star}\sum_{i\in \rbs_t}(\arms+\J^\star)\cdot
 \min\sset{\sqrt{8\zeta\sigma_i^2\log t \over\counter{i}{t-1}{\hidervectorm}} + 
 {13.3\zeta\log t \over\counter{i}{t-1}{\hidervectorm}}-{J^\star\gapss{\rbs_t}\over 2\arms(\arms+\J^\star)}\CommaBin1}+2\Delta_{\costvectorm}\!\pa{\rbs_t}\\ 
 &\leq {2\over \J^\star}\sum_{i\in \rbs_t}(\arms+\J^\star)\cdot
 \min\sset{\sqrt{8\zeta\sigma_i^2\log t \over\counter{i}{t-1}{\hidervectorm}} \CommaBin1}\nonumber
 \\&\quad+
 {2\over \J^\star}\sum_{i\in[\abs{\rbs_t}]}\min\sset{\sqrt{2\zeta\log t \over \counter{\rs_{i,t}}{t-1}{\costvectorm}},1} \left(1-\hidervector\transpose\be_{\rbs_t[i-1]}\right)\!.\label{eq:eventa}
\end{align}
Where \eqref{eq:eventn} uses event $\mathfrak{N}_t$ and $\abs{\bs_t}\leq \arms$, \eqref{eq:eventa} uses event $\mathfrak{A}_t$.
\subsection{Use of \citet{wang2017improving} results}

From this point, since $\left(1-\hidervector\transpose\be_{\rbs_t[i-1]}\right)$ is the probability of getting cost feedback from arm $i$ at round $t$, the analysis given by Theorem~1 of \cite{wang2017improving} takes care of the second term, while the analysis of their Theorem 4 takes care of the first. We restate their results in Theorem~\ref{wang2} and~\ref{wang1}, respectively. We want to use these results with $\mathfrak{B}_t$ being the intersection of events $\mathfrak{A}_t,\mathfrak{M}_t,\mathfrak{N}_t $, and with $M_i=\Delta_{i,\min}$.
On the one hand, we apply second result of each theorem, for the first with 
\[\lambda={2(\arms+\J^\star)\over \J^\star}~\text{and for all }i\in[\arms],~\lip_i=2\sqrt{\frac{\zeta\sigma_i^2}{{3}}}\]
and for the second with
\[\lambda={2\over \J^\star}~\text{and for all }i\in[\arms],~\lip_i=\sqrt{\frac{\zeta}{3}}\cdot\]
We thus get, using that $\sum_{i\in[\arms]}\sigma_i^2\leq1,$ 

\begin{align*}\EE{\sum_{t=1}^{\rd\tau_B -1}\gapss{s_t}\II{\mathfrak{A}_t,\mathfrak{M}_t,\mathfrak{N}_t}}&\leq \frac{2}{\J^\star}
~\E\pa{\sum_{t=1}^{\rd\tau_B -1}\pa{\sizecorr{ \sum_{i\in[\abs{\rbs_t}]}\min\sset{\sqrt{2\zeta\log t \over \counter{\rs_{i,t}}{t-1}{\costvectorm}},1} \pa{1-\hidervector\transpose\be_{\rbs_t[i-1]}}}\sum_{i\in \rbs_t}(\arms+\J^\star)\cdot
 \min\sset{\sqrt{8\zeta\sigma_i^2\log t \over\counter{i}{t-1}{\hidervectorm}} \CommaBin1} 
 \right.\right.
\\&\quad+
\left.\left.
 \sum_{i\in[\abs{\rbs_t}]}\min\sset{\sqrt{2\zeta\log t \over \counter{\rs_{i,t}}{t-1}{\costvectorm}},1} \pa{1-\hidervector\transpose\be_{\rbs_t[i-1]}}
 }}
\\&\leq
{\sqrt{\zeta} \over {\J^\star}}\pa{ 32.4(\J^\star+\arms)
 +13.9\sqrt{\arms}}\EE{\sqrt{\arms\pa{\tau_B-1}\log\pa{\rd\tau_B-1}} }
 \\&\quad+
{\pi^2\arms^2\over 3\J^\star}\E \ceil{\log_2\left({{\rd\tau_B-1}\over 18 \log\pa{{\rd\tau_B-1}} 
}\right)}^++{4\arms\pa{1 + \arms+\J^\star}\over \J^\star}\cdot
\end{align*}

\noindent On the other hand, we can multiply \eqref{eq:eventa} by $4$, to get that $2\gapss{\rbs_t}$ is bounded by $A_t+B_t$, where \[ A_t={8(\arms+\J^\star)\over \J^\star}\sum_{i\in \rbs_t}
 \min\sset{\sqrt{8\zeta\sigma_i^2\log t \over\counter{i}{t-1}{\hidervectorm}} ,1}-\sup_{i\in \rbs_t}\Delta_{i,\min}\]
 and
 \[B_t=
 {8\over \J^\star}\sum_{i\in[\abs{\rbs_t}]}\min\sset{\sqrt{2\zeta\log t \over \counter{\rs_{i,t}}{t-1}{\costvectorm}},1} \left(1-\hidervector\transpose\be_{\rbs_t[i-1]}\right)- \sup_{i\in \rbs_t}\Delta_{i,\min}.\]
\noindent
 We then apply first result of each theorem. Theorem~\ref{wang1} is applied to $A_t$ with
\[\lambda={4(\arms+\J^\star)\over \J^\star}~\text{and for all }i\in[\arms],~\lip_i=2\sqrt{\frac{\zeta\sigma_i^2}{{3}}}\]
and  Theorem~\ref{wang2} is applied to $B_t$ with
\[\lambda={4\over \J^\star}~\text{and for all }i\in[\arms],~\lip_i=\sqrt{\frac{\zeta}{3}}\cdot\]
We thus finally get
 \begin{align*}\EE{\sum_{t=1}^{{\rd\tau_B-1}}\gapss{s_t}\II{\mathfrak{A}_t,\mathfrak{M}_t,\mathfrak{N}_t}}
 =& 
 \frac{1}{2}\EE{\sum_{t=1}^{{\rd\tau_B-1}}2\gapss{s_t}\II{\mathfrak{A}_t,\mathfrak{M}_t,\mathfrak{N}_t}}
 \\\leq&
 \frac{1}{2}\EE{\sum_{t=1}^{{\rd\tau_B-1}}\pa{A_t+B_t}\II{\mathfrak{A}_t,\mathfrak{M}_t,\mathfrak{N}_t}}
 \\\leq& 
{1\over {\J^\star}^2}\sum_{i\in [\arms]}\left({512\zeta\sigma_i^2\arms(\J^\star+\arms)^2\over\Delta_{i,\min} }+{1536\zeta\arms \over \Delta_{i,\min}}  
 \right)\EE{\log({\rd\tau_B-1})} \\&\quad+
{\pi^2\arms\over 3\J^\star}\sum_{i\in [\arms]} \ceil{\log_2\left({8\arms\over \J^\star
\Delta_{i,\min}}\right)}^++{8\arms + 4\arms(\arms+\J^\star)\over \J^\star} \cdot
\end{align*}

\subsection{Regret bound from low probability events} 
 
 Here, we are going to bound the regret under the event $\neg \pa{\mathfrak{A}_t\cap\mathfrak{M}_t\cap\mathfrak{N}_t}$. 

 \noindent By Hoeffding's inequality (Fact~\ref{hoeffding}), and Theorem~1 of \citet[Fact~\ref{thm:audibert}]{audibert2009}, we have that $\mathfrak{M}_t$ holds with probability at least $1-4\arms\beta(t)$. $\mathfrak{N}_t$ holds with probability at least $1-3\arms\beta(t)$ by Hoeffding's inequality (Fact~\ref{hoeffding}), and Lemma~\ref{lem:variance}. Thus, since $\gapss{\rbs_t}\leq \nicefrac{\arms}{\J^\star},$

 \begin{align*}\EE{\sum_{t=1}^{\rd\tau_B-1}\gapss{\rbs_t}\pa{\II{\neg \mathfrak{N}_t}+\II{\neg \mathfrak{M}_t}}}&\leq  \frac{7\arms^2}{\J^\star}\sum_{t>0}\beta(t).\end{align*}
 By tedious computations, $\sum_{t>0}\beta(t)$ can be bounded by $786$.
 
 The upper bound under event $\neg \mathfrak{A}_t$ uses the following proposition~\ref{lemma5chen}.
 \begin{proposition}
 Let $\pa{\ell_t}_t$ be an increasing sequence of decreasing functions. For any fixed arm $i\in [\arms]$, we define $n_i$ as the number of searches that contains arm $i$. Let $\Delta_{i,1}\geq \dots\geq \Delta_{i,n_i}$ be the gaps of these actions.
 For any random horizon $\tau$, we have
 $$\sum_{t\in [\tau]}\gapss{\rbs_t}\II{i\in \bs_t, \counter{i}{t-1}{\hidervectorm}\leq \ell_t\pa{\gapss{\rbs_t}} }\leq \ell_\tau\pa{\Delta_{i,1}}\Delta_{i,1} + \int_{\Delta_{i,n_i}}^{\Delta_{i,1}}\ell_\tau\pa{x}dx .$$
  \label{lemma5chen}
 \end{proposition}
\begin{proof}
We let $\Delta_{i,0}=\infty.$ In the following equations, small changes between successive lines are highlighted in red. \begin{align*}&\sum_{t\in [\tau]}\gapss{\rbs_t}\II{i\in \bs_t, \counter{i}{t-1}{\hidervectorm}\leq \ell_t\pa{\gapss{\rbs_t}} }
\\&\leq
\sum_{t\in [\tau]}\gapss{\rbs_t}\II{i\in \bs_t, \counter{i}{t-1}{\hidervectorm}\leq\textcolor{red}{ \ell_\tau}\pa{\gapss{\rbs_t}} }
\\&=
\sum_{t\in [\tau]}\textcolor{red}{\sum_{k\in[n_i]}\Delta_{i,k}}\II{i\in \bs_t, \counter{i}{t-1}{\hidervectorm}\leq \ell_\tau\pa{\textcolor{red}{\Delta_{i,k}}}, \textcolor{red}{\gapss{\rbs_t}=\Delta_{i,k}} }
\\&=
\sum_{t\in [\tau]}\sum_{k\in[n_i]}\textcolor{red}{\sum_{j\in [k]}}\Delta_{i,k}\II{i\in \bs_t,\textcolor{red}{\ell_\tau\pa{\Delta_{i,j-1}} <\counter{i}{t-1}{\hidervectorm} \leq \ell_\tau\pa{\Delta_{i,j}}}, \gapss{\rbs_t}=\Delta_{i,k} }
\\&\leq
\sum_{t\in [\tau]}\sum_{k\in[n_i]}{\sum_{j\in [k]}}\textcolor{red}{\Delta_{i,j}}\II{i\in \bs_t,{\ell_\tau\pa{\Delta_{i,j-1}} <\counter{i}{t-1}{\hidervectorm} \leq \ell_\tau\pa{\Delta_{i,j}}}, \gapss{\rbs_t}=\Delta_{i,k} }
\\&\leq
\sum_{t\in [\tau]}\sum_{k\in[n_i]}\textcolor{red}{\sum_{j\in [n_i]}}{\Delta_{i,j}}\II{i\in \bs_t,{\ell_\tau\pa{\Delta_{i,j-1}} <\counter{i}{t-1}{\hidervectorm} \leq \ell_\tau\pa{\Delta_{i,j}}}, \gapss{\rbs_t}=\Delta_{i,k} }
\\&\leq
\sum_{t\in [\tau]}{\sum_{j\in [n_i]}}{\Delta_{i,j}}\II{i\in \bs_t,{\ell_\tau\pa{\Delta_{i,j-1}} <\counter{i}{t-1}{\hidervectorm} \leq \ell_\tau\pa{\Delta_{i,j}}}, \textcolor{red}{\gapss{\rbs_t}>0} }
\\&\leq
\ell_\tau\pa{\Delta_{i,1}}\Delta_{i,1} + \sum_{j=2}^{ n_i}\Delta_{i,j}\pa{\ell_\tau\pa{\Delta_{i,j}}-\ell_\tau\pa{\Delta_{i,j-1}}}
\\&\leq
\ell_\tau\pa{\Delta_{i,1}}\Delta_{i,1} + \int_{\Delta_{i,n_i}}^{\Delta_{i,1}}\ell_\tau\pa{x}dx .
\end{align*}
\end{proof}
According to event $\mathfrak{A}_t$, we want to use Proposition~\ref{lemma5chen} with \begin{align}\label{defell}\ell_t(x)\triangleq\frac{26.6\zeta\log(t)\arms\pa{\arms+\J^\star}}{\J^\star x}\cdot\end{align}
This gives
\begin{align*}
 \sum_{t=1}^{\rd\tau_B-1}\gapss{\rbs_t}{\II{\neg \mathfrak{A}_t}}&\leq
 \sum_{i\in [\arms]}\ell_{\tau_B-1}\pa{\Delta_{i,1}}\Delta_{i,1} + \int_{\Delta_{i,n_i}}^{\Delta_{i,1}}\ell_{\tau_B-1}\pa{x}dx
 \\&\leq
\frac{26.6\zeta\arms\pa{\arms+\J^\star}}{\J^\star }\sum_{i\in[\arms]}\pa{1+\log\pa{\frac{\arms}{\J^\star\Delta_{i,\min}}}}\log(\tau_B-1).
\end{align*}
Notice, this bound depends on the gap $\Delta_{i,\min}$. However, for the gap-free upper bound, this term is negligible:  \begin{align*}\forall\Delta>0,\quad\sum_{t=1}^{\tau_B-1}\II{\neg \mathfrak{A}_t}\gapss{\bs_t}\leq& \sum_{t=1}^{\tau_B-1}\II{\neg \mathfrak{A}_t,~\gapss{\bs_t}\leq \Delta}\Delta \\&+\frac{26.6\zeta\arms\pa{\arms+\J^\star}}{\J^\star }\sum_{i\in[\arms]}\pa{1+\log\pa{\frac{\arms}{\J^\star\Delta}}}\log(\tau_B-1)\\\leq&\pa{\tau_B-1}\Delta+\frac{26.6\zeta\arms\pa{\arms+\J^\star}}{\J^\star }\sum_{i\in[\arms]}\pa{1+\log\pa{\frac{\arms}{\J^\star\Delta}}}\log(\tau_B-1).\end{align*}
Taking $\Delta=\pa{\tau_B-1}^{-1}$ gives a term growing as $\log^2(\tau_B-1)$, negligible compared to $\sqrt{\tau_B-1}\log(\tau_B-1)$.
\subsection{Control on the random time horizon}

The know techniques to tackle the random horizon in regret upper bounds \citep{pmlr-v45-Xia15,xia2016budgeted} are based on the following variant of Hoeffding's inequality, in order to bound $\rd \tau_B$ by $T_B$ with high probability.
\begin{fact}[\citealp{hoeffding1963probability, flajolet2015}]\label{hoeffdingbis}
Let $\rd{X}_1, \dots , \rd{X}_t$ be the random variables with common support $[0, 1]$ and such that there exists $a\in \R$ with $\forall u\in [t],~\EE{\rd{X}_u|\rd{X}_1,\dots ,\rd{X}_{u-1}}\geq a$. Let $\rd{\bar{x}}_t\triangleq{1\over t}(\rd{X}_1+ \dots + \rd{X}_t)$, then \[\forall \varepsilon\geq 0~,~\PP{\rd{\bar{x}}_t-a \leq -\varepsilon}\leq e^{-2\varepsilon^2t}.\]
\end{fact} 
Indeed, this would decompose $\EE{\sum_{t=1}^{\rd \tau_B -1}\Delta\pa{\rbs_t}}$ into a term with deterministic horizon $\EE{\sum_{t=1}^{T_B}\Delta\pa{\rbs_t}}$, and another of order $e^{-\costm_{\min}B}/ {\costm^2_{\min}}.$ Although the second term decreases exponentially fast to $0$ when $B\to \infty$, the dependence on $\nicefrac{1}{{\costm^2_{\min}}}$ is undesirable and artificial. Here, we rather keep the random horizon inside expected regret upper bounds on $\EE{\sum_{t=1}^{\rd \tau_B -1}\Delta\pa{\rbs_t}}$, as we did above. Indeed, these bounds are quantities having a factor $\EE{f\pa{\rd \tau_B-1}}$, $f$ being an increasing concave function. Thus, we have, with Jensen's inequality, upper bounds with a factor $f\pa{\EE{\rd \tau_B-1}}$. We provide in the following a control on $\EE{\rd \tau_B-1}$.
\begin{align}
\nonumber\EE{\rd\tau_B}=1+\EE{\sum_{t\geq 1}\II{\rd{B}_{t}\geq 0}}&=1+
\sum_{t\geq 1}\PP{{B-t\costm_{\min}+t\costm_{\min}\geq \sum_{u=1}^t\rcostvector_u\transpose\be_{\rs_u[\rhidervector_u]}}}\\
&\leq
T_B+1+\sum_{t\geq T_B+1}\exp\left({{-2(B-t\costm_{\min})^2\over t}}\right) 
\label{flajolet}\\\label{poly}
&\leq
T_B+1+\sum_{t\geq T_B+1}\exp\left({{-{\costm^2_{\min}} t\over 2}}\right)\\\label{serie}
 &\leq 
 T_B+1+{2\over {\costm^2_{\min}}}\exp\left({{{{\costm^2_{\min}}\over 2}-\frac{{\costm^2_{\min}}(T_B+1)}{2}}}\right) \\&\leq 
T_B+1+{2\over  {\costm^2_{\min}}}\exp\left({{-{\costm_{\min}B}}}\right)\!,\label{bound_on_tau}
\end{align}
where \eqref{flajolet} makes use of Fact~\ref{hoeffdingbis}, \eqref{poly} is obtained because $2(B-t\costm_{\min})^2\geq {{\costm^2_{\min}} t^2/ 2}$ for $t\geq {2B/ \costm_{\min}}$ 
and we get \eqref{serie}  since $1/( 1-e^{-{\costm^2_{\min}}/2})\leq 2e^{{\costm^2_{\min}}/2}/{\costm^2_{\min}}$. 

Notice, the gain using Jensen's inequality for an upper bound with $f(x)=\sqrt{x}\log(x)$ (i.e., when we want a gap-free upper bound on the regret) is a change from a term of order $\nicefrac{1}{{\costm^2_{\min}}}$ to a term of order $\nicefrac{1}{{\costm_{\min}}}$. The gain is more
relevant for $f=\log$ (i.e., when we want a logarithmic bound on the regret), since \begin{align*}\log\pa{\EE{\tau_B-1}}&\leq\log\pa{T_B+{2\over  {\costm^2_{\min}}}\exp\pa{-\costm_{\min}B}}\\&\leq \log\pa{T_B+{4B^2\over  {\costm^2_{\min}}}}&\text{since }\exp\pa{-\costm_{\min}B}\leq 1\leq  2 B^2\\
&\leq \log\pa{T_B+T_B^2}\leq 2\log\pa{T_B+1/2} .\end{align*}

\subsection{\citet{wang2017improving} results}

We built on the results of~\citet{wang2017improving} 
for combinatorial multi-armed bandits with probabilistically triggered arms (CMAB-T). 
In particular, \citet{wang2017improving} give  expected regret bounds under specific assumptions that our setting satisfies.
In CMAB-T, at each round $t$, the agent selects some action $\rbs_t$ and a random subset of \emph{arms}
is triggered.  The corresponding feedback is given to the agent which then goes to the next round.
We denote by $\sigma(\rbs_t)$, the set of arms that have a positive probability of being triggered if $\rbs_t$ is selected, and $\rd{\sigma}(\rbs_t)\subset\sigma(\rbs_t)$ the random subset of arms~$i$ that are actually triggered and for which we maintain a counter $\rd N_{i,t}$.
We restate two results of \citet{wang2017improving} that hold under following assumptions. Notice that we generalize their results to a random horizon.
 For a round $t\geq 1$, we let $\mathfrak{B}_t$ be any event.
 We let $\mathbf{M}\in (0,\infty)^\arms$ and for an action $\bs$, $M\!\pa{\bs}=\sup_{i\in \sigma(\bs)}M_i$. We let $\rd \tau$ a (possibly random) round, and $\mathbf{\lip}\in \mathbb{R}_+^\arms$.
 $H\!\pa{\bs}$ and $\lambda$ are non-negative numbers, $H\!\pa{\bs}$ (deterministically) depends on $\bs$. We write $\rbs_t$ the action chosen at round $t$. 
\begin{theorem}
Suppose that $\forall  \bs, \forall i \in \sigma(\bs),~ \PP{i\in \rd{\sigma}(\bs)}=1$. 
If\, for all $t$, under event $\mathfrak{B}_t$, 
$$H\!\pa{\bs_t}\leq \sum_{i\in \sigma(\rbs_t)}\lambda\min\sset{ 2\lip_i\sqrt{1.5\log t \over N_{i,t-1}}\CommaBin1 }
 \CommaBin$$ then
\begin{equation*}
\EE{\sum_{t=1}^{\rd\tau} \pa{2H\!\pa{\bs_t}-M\!\pa{\bs_t}}\II{\mathfrak{B}_t}}\leq \sum_{i\in [\arms]}\EE{48\arms\lip_i^2 \lambda^2\log(\rd \tau)\over M_i} 
+2\lambda\arms
\quad\text{and}
\end{equation*}

\begin{equation*}
\EE{\sum_{t=1}^{\rd\tau} H\!\pa{\bs_t}\II{\mathfrak{B}_t}}\leq 14\lambda\norm{\mathbf{\lip}}_2\EE{\sqrt{\arms \rd \tau\log(\rd \tau)}}
+2\lambda\arms.
\end{equation*}
 \label{wang1}
\end{theorem}

\begin{theorem}
If for all $t$, under event $\mathfrak{B}_t$,  
$$H\!\pa{\bs_t}\leq \sum_{i\in \sigma(\rbs_t)}\PP{i\in \sigma(\rbs_t)}\lambda\min\sset{ 2\lip_i\sqrt{1.5\log t \over N_{i,t-1}}\CommaBin1 }\CommaBin$$ then
\begin{equation*}
\EE{\sum_{t=1}^{\rd\tau} \pa{2H\!\pa{\bs_t}-M\!\pa{\bs_t}}\II{\mathfrak{B}_t}}\leq \sum_{i\in [\arms]}\EE{{576\arms\lip_i^2 \lambda^2\log({\rd\tau})\over M_i}}+ 
{\pi^2\lambda\arms\over 6}\sum_{i\in [\arms]} \ceil{\log_2\left({2\lambda\arms\over  M_i}\right)}^+ 
+4\arms\lambda\quad\text{and}
\end{equation*}

\begin{equation*}
\EE{\sum_{t=1}^{\rd\tau } H\!\pa{\bs_t}\II{\mathfrak{B}_t}}\leq 12\lambda\norm{\mathbf{\lip}}_2\EE{\sqrt{\arms \rd \tau\log(\rd \tau)}}+ 
{\pi^2\lambda\arms^2\over 6}\E\ceil{\log_2\left({{\rd\tau}\over  18\log({\rd\tau})}\right)}^+ 
+2\arms\lambda.
\end{equation*}
 \label{wang2}
\end{theorem}
 \end{proof}

\section{Proof of Theorem~\ref{lower}}
\label{app:lower}

\setcounter{scratchcounter}{\value{theorem}}\setcounter{theorem}{\the\numexpr\getrefnumber{lower}-1}\restaL*\setcounter{theorem}{\the\numexpr\value{scratchcounter}}

\setcounter{scratchcounter}{\value{figure}}\setcounter{figure}{\the\numexpr\getrefnumber{2paths}-1}\restaM*\setcounter{figure}{\the\numexpr\value{scratchcounter}}

\begin{proof}
Let $0<\varepsilon<1/4$. We consider a DAG composed of two disjoint paths (Figure~\ref{2paths}), both with $\arms/2$ nodes.
We denote the two paths by $\ba$ and $\bb$. We deterministically set all the costs  to $1$, $w_i=0\text{ for } i\notin\sset{a_{\frac{\arms}{2}},b_{\frac{\arms}{2}}}$. All this information is given to the agent.
Notice that this does not make the problem harder.

 Now consider two distributions $\distribution_1$ and $\distribution_2$ defined by  \begin{align*}\distribution_1:\quad w_{a_{\frac{\arms}{2}}}\triangleq\frac{1}{2}+\varepsilon,\quad w_{b_{\frac{\arms}{2}}}\triangleq\frac{1}{2}-\varepsilon
\quad\text{and}\quad\distribution_2:\quad w_{a_{\frac{\arms}{2}}}\triangleq\frac{1}{2}-\varepsilon,\quad w_{b_{\frac{\arms}{2}}}\triangleq\frac{1}{2}+\varepsilon.
\end{align*}
Notice that an optimal online policy does not modify its behavior during a round $t$, since after having seen $\rhider_{i,t}=1$, continuing searching would only give information about cost distribution which is known by the problem definition, and no additional information about the rewards. Therefore, there is \emph{an} optimal online policy that selects some search $\bs$ and perform $\bs[\rhidervector_t]$ over round $t$.
Observe that $\bs^\star=\ba\bb$  for  $\distribution_1$ and  $\bs^\star=\bb\ba$ for $\distribution_2$. We have $\J^\star=\frac{3}{4}\arms-\varepsilon\arms\geq \frac{1}{2}\arms$ for both $\distribution_1$ and~$\distribution_2$. 

We now show that we can restrict ourselves to policies that take searches in $\sset{\ba\bb,\bb\ba}$.
\begin{itemize}
 \item
First, an optimal online policy does not select a search that would not include at least one of the leaves $\sset{a_{\frac{\arms}{2}},b_{\frac{\arms}{2}}}$  for a round. Therefore, it has a full information on $\rhider$. Indeed, such a search is noninformative and does not bring any reward while having a cost. 
\item Second, for a policy $\pi$ that does not select a search in $\sset{\ba\bb,\bb\ba}$ for some round $t$, we construct $\pi'$ that acts like~$\pi$ except for this round $t$ where it selects $\ba\bb$ if $\pi$ would see the leaf $a_{\frac{\arms}{2}}$ first, and $\bb\ba$ otherwise, i.e., if $\pi$ would first see the leaf $b_{\frac{\arms}{2}}$. Now compare both policies on the same realization of $\rhidervector_1,\rhidervector_2,\dots$. 
We claim that the global reward of $\pi'$ is  never smaller than that of $\pi$.
By symmetry, assume that $\pi$ sees $a_{\frac{\arms}{2}}$ first within round $t$ and thus $\pi'$ selects $\ba\bb$. 
\begin{itemize}
\item If $\rhider_{a_{\frac{\arms}{2}},t}=1$ or $\pa{\rhider_{b_{\frac{\arms}{2}},t}=1\text{ and }\pi \text{ visits }b_{\frac{\arms}{2}}\text{ within round }t}$, both policies obtain the same reward of $1$ within round $t$, but $\pi'$ pays less than $\pi$.
\item If $\rhider_{b_{\frac{\arms}{2}},t}=1$ and $\pi$ does not visit $b_{\frac{\arms}{2}}$ within round $t$, $\pi$ gains $0$ and pays at least $\arms/2$, whereas $\pi'$ gains $1$ and pays $\arms$ within round $t$. Thus, the budget of $\pi$ compared to $\pi'$ is augmented by at most $\arms/2$, with which it can increase its reward by at most $1$.
\end{itemize}
The overall reward of $\pi'$ remains higher than that of $\pi$ for both cases. 
\end{itemize}
 
 A direct consequence of the restriction to  $\sset{\ba\bb,\bb\ba}$ is that $\costm_{\min}= \arms/2$, giving the upper bound in Theorem~\ref{lower} by invoking the result of  Theorem~\ref{banditregret}.  

Now for a policy $\pi$ using searches from $\sset{\ba\bb,\bb\ba}$, we have \begin{align*}
      F_B(\pi)=\sum_{t=1}^\infty&\EE{ \sum_{i\in \bs_t}\II{\rd{B}_t\geq 0,~\rhider_{i,t}=1}}
\leq
\sum_{t=1}^\infty\EE{ \sum_{i\in \bs_t}\II{\rd{B}_{t-1}\geq 0,~\rhider_{i,t}=1}}\\
\nonumber &= \sum_{t\geq 1}\EE{\II{\rd{B}_{t-1}\geq 0,~\rbs_t\in \search^\star}\be_{\rbs_t}\transpose\hidervector}
+
\sum_{t\geq 1}\EE{\II{\rd{B}_{t-1}\geq 0,~\rbs_t\notin \search^\star}\be_{\rbs_t}\transpose\hidervector}
\\
\nonumber&= {1\over \J^\star}\sum_{t\geq 1}\EE{\II{\rd{B}_{t-1}\geq 0,~\rbs_t\in \search^\star}\left({\D(\rbs_t)}+(1-\be_{\rbs_t}\transpose\hidervector)\be_{\rbs_t}\transpose\costvector\right)}
\\&\quad+
\nonumber{1\over \J^\star}\sum_{t\geq 1}\EE{\II{\rd{B}_{t-1}\geq 0,~\rbs_t\notin \search^\star}{\left({\D(\rbs_t)}+(1-\be_{\rbs_t}\transpose\hidervector)\be_{\rbs_t}\transpose\costvector\right)}}
-
\nonumber\sum_{t\geq 1}\EE{\II{\rd{B}_{t-1}\geq 0,~\rbs_t\notin \search^\star}\Delta\pa{\rbs_t}}
\\&=\nonumber
{1\over \J^\star}\sum_{t\geq 1}\EE{\II{\rd{B}_{t-1}\geq 0}{{\left({\D(\rbs_t)}+(1-\be_{\rbs_t}\transpose\hidervector)\be_{\rbs_t}\transpose\costvector\right)}}}
-
\sum_{t\geq 1}\EE{\II{\rd{B}_{t-1}\geq 0,~\rbs_t\notin \search^\star}\Delta\pa{\rbs_t}}
\\&\leq {B+\arms \over \J^\star} -
\nonumber\sum_{t\geq 1}\EE{\II{\rd{B}_{t-1}\geq 0,~\rbs_t\notin \search^\star}\Delta\pa{\rbs_t}}\!.
\end{align*}
As a result we get
 \begin{align*}
\RB_{B}({\pi})&=F_B^\star-F_B({\pi})\geq \frac{B-\arms}{\J^\star}-\frac{B+\arms}{\J^\star}+\sum_{t\geq 1}\EE{\II{\rd{B}_{t-1}\geq 0,~\rbs_t\notin \search^\star}\Delta\pa{\rbs_t}}\\&=-\frac{2\arms}{\J^\star}+\sum_{t\geq 1}\EE{\II{\rd{B}_{t-1}\geq 0,~\rbs_t\notin \search^\star}\Delta\pa{\rbs_t}}\cdot
\end{align*}
Since we restrict $\pi$ to take a search in $\sset{\ba\bb,\bb\ba}$, we have a single gap (the same for $\distribution_1$ and $\distribution_2$) \begin{align}\Delta=\frac{\frac{\arms}{2}\pa{\frac{1}{2}-\varepsilon}+\arms\pa{\frac{1}{2}+\varepsilon} }{\frac{\arms}{2}\pa{\frac{1}{2}+\varepsilon}+\arms\pa{\frac{1}{2}-\varepsilon}}-1=\frac{1.5+\varepsilon}{1.5-\varepsilon}-1=\frac{2\varepsilon}{1.5-\varepsilon}\geq \frac{4\varepsilon}{3}\cdot\label{monogap}\end{align}
Furthermore we can bound the number of rounds from below  by $B/\arms$. To proceed we use high-probability Pinsker inequality \cite[Lemma 2.6]{tsybakov2009introduction}.
\begin{fact}[high-probability Pinsker inequality]\label{HPpinsker}
 Let $P$ and $Q$ be probability measures on the same measurable space, and let $\cA$ be an event. Then \[P(\cA)+Q(\neg \cA)\geq \frac{1}{2}\exp\pa{-\KL{P}{Q}}\!,\]
 where $\text{KL}$ is the Kullback-Leibler divergence.
\end{fact}

We let $\RB_{1,B}({\pi})$ be the regret of $\pi$ for distribution $\distribution_1$ and similarly, $\RB_{2,B}({\pi})$ for $\distribution_2$. If $\PPno_1$ and $\PPno_2$ denote the probability when random variable are samples from $\distribution_1$ and $\distribution_2$ respectively, we have
\begin{align}\nonumber\max\sset{\RB_{1,B}({\pi}),\RB_{2,B}({\pi})}&\geq \frac{\RB_{1,B}({\pi})+\RB_{2,B}({\pi})}{2}\\&\nonumber\geq -\frac{2\arms}{\J^\star}+\frac{\Delta}{2}\sum_{t=1}^{B/\arms}\pa{{\PPt{1}{\rd{B}_{t-1}\geq 0,\bs_t=\bb\ba}+\PPt{2}{\rd{B}_{t-1}\geq 0,\bs_t=\ba\bb}}}
\\&\geq -\frac{2\arms}{\J^\star}+\frac{\varepsilon}{3}\sum_{t=1}^{B/\arms}\exp\pa{-\KL{\distribution_1^{\otimes t}}{\distribution_2^{\otimes t}}}\label{kl}
\\&=  -\frac{2\arms}{\J^\star}+\frac{\varepsilon}{3}\sum_{t=1}^{B/\arms}\exp\pa{-t\KL{\distribution_1}{\distribution_2}}\!,\nonumber
\end{align} 
where \eqref{kl} is due to Fact~\ref{HPpinsker} and \eqref{monogap}. Then, \begin{align}\KL{\distribution_1}{\distribution_2}&=\pa{\frac{1}{2}+\varepsilon}\log\pa{\frac{\frac{1}{2}+\varepsilon}{\frac{1}{2}-\varepsilon}}+\pa{\frac{1}{2}-\varepsilon}\log\pa{\frac{\frac{1}{2}-\varepsilon}{\frac{1}{2}+\varepsilon}}\nonumber\\&\leq 2\varepsilon\pa{\frac{\frac{1}{2}+\varepsilon}{\frac{1}{2}-\varepsilon}}-2\varepsilon\pa{\frac{\frac{1}{2}-\varepsilon}{\frac{1}{2}+\varepsilon}}=\frac{4\varepsilon^2}{{\frac{1}{4}-\varepsilon^2}}\leq \frac{64}{3}\varepsilon^2 & (\text{because\ } \log(x)\leq x-1).\nonumber\end{align}
Thus, with $\J^\star\geq \frac{\arms}{2}$, we have   \begin{align*}\max\sset{\RB_{1,B}({\pi}),\RB_{2,B}({\pi})}\geq-4+\frac{\varepsilon}{3}\sum_{t=1}^{B/\arms}\exp\pa{-\frac{64}{3}t\varepsilon^2}&\geq-4+\frac{\varepsilon\pa{1-\exp\pa{-\frac{64}{3}\frac{B\varepsilon^2}{\arms}}}}{3\pa{\exp\pa{\frac{64}{3}\varepsilon^2}-1}}\\&\geq-4+\frac{{1-\exp\pa{-\frac{64}{3}\frac{B\varepsilon^2}{\arms}}}}{64\varepsilon}\\&\geq-4+\min\sset{\frac{1}{128\varepsilon}\CommaBin\frac{\varepsilon B}{6\arms}} .\end{align*}
Taking $\varepsilon=\sqrt{(6\arms)/(128B)}$, the lower bound becomes $$\max\sset{\RB_{1,B}({\pi}),\RB_{2,B}({\pi})}\geq-4+\sqrt{\frac{B}{768\arms}}\geq -4+\frac{1}{28}\sqrt{\frac{B}{\arms}}\cdot$$
\end{proof}
%

\end{document}

%% file: misosymbols.tex
\definecolor{graphicbackground}{rgb}{0.96,0.96,0.8}
\definecolor{rouge1}{RGB}{226,0,38}  
\definecolor{orange1}{RGB}{243,154,38}  
\definecolor{jaune}{RGB}{254,205,27}  
\definecolor{blanc}{RGB}{255,255,255} 
\definecolor{rouge2}{RGB}{230,68,57}  
\definecolor{orange2}{RGB}{236,117,40}  
\definecolor{taupe}{RGB}{134,113,127} 
\definecolor{gris}{RGB}{91,94,111} 
\definecolor{bleu1}{RGB}{38,109,131} 
\definecolor{bleu2}{RGB}{28,50,114} 
\definecolor{vert1}{RGB}{133,146,66} 
\definecolor{vert3}{RGB}{20,200,66} 
\definecolor{vert2}{RGB}{157,193,7} 
\definecolor{darkyellow}{RGB}{233,165,0}  
\definecolor{lightgray}{rgb}{0.9,0.9,0.9}
\definecolor{darkgray}{rgb}{0.6,0.6,0.6}
\definecolor{babyblue}{rgb}{0.54, 0.81, 0.94}
\definecolor{citrine}{rgb}{0.89, 0.82, 0.04}
\definecolor{misogreen}{rgb}{0.25,0.6,0.0}

\newcommand{\sset}[1]{\left\{#1\right\}}
\newcommand{\ceil}[1]{\left\lceil#1\right\rceil}

\newcommand{\II}[1]{\mathbb{I}{\left\{#1\right\}}}

\newcommand{\Bernoulli}{\mathrm{Bernoulli}}

 \newtheorem{lemma}{Lemma}
 \newtheorem{theorem}{Theorem}
 
 \newtheorem{definition}{Definition}
 \newtheorem{proposition}{Proposition}
\newtheorem{fact}{Fact}
 \newtheorem{example}{Example}
\newcommand{\argmin}{\text{argmin}}
\newcommand{\R}{\mathbb{R}}

\newcommand{\E}{\mathbb{E}}
\newcommand{\EE}[1]{\mathbb{E}\left[#1\right]}

\newcommand{\PP}[1]{\mathbb{P}\left[#1\right]}
\newcommand{\PPno}{\mathbb{P}}
\newcommand{\PPt}[2]{\mathbb{P}_{#1}\left[#2\right]}

\newcommand{\pa}[1]{\left(#1\right)}

\newcommand{\norm}[1]{\left\|#1\right\|}

\newcommand{\abs}[1]{\left|#1\right|}

\let\originalleft\left
\let\originalright\right
\renewcommand{\left}{\mathopen{}\mathclose\bgroup\originalleft}
\renewcommand{\right}{\aftergroup\egroup\originalright}

\newcommand{\CommaBin}{\mathbin{\raisebox{0.5ex}{,}}}

\newcommand{\transpose}{^\mathsf{\scriptscriptstyle T}}

\newcommand{\cA}{\mathcal{A}}

\newcommand{\cE}{\mathcal{E}}

\newcommand{\cJ}{\mathcal{J}}

\newcommand{\cO}{\mathcal{O}}

\newcommand{\OO}{\mathcal{O}}

\newcommand{\ba}{{\bf a}}

\newcommand{\bb}{{\bf b}}

\newcommand{\be}{{\bf e}}

\newcommand{\bs}{{\bf s}}

\newcommand{\by}{{\bf y}}
\newcommand{\bx}{{\bf x}}

\renewcommand{\epsilon}{\varepsilon}

\renewcommand{\bar}{\overline}

\newcommand{\nothere}[1]{}

\usepackage{xspace}

\newcommand{\UCB}{\textsc{UCB}\xspace}
\newcommand{\UCBV}{\textsc{UCB-V}\xspace}

\newcommand{\CUCB}{\normalfont \textsc{CUCB}\xspace}
\newcommand{\CUCBV}{\normalfont \textsc{CUCB-V}\xspace}
\newcommand{\CUCBKL}{\normalfont \textsc{CUCB-KL}\xspace}

\newcommand{\ThompsonSampling}{\textsc{ThompsonSampling}\xspace}

\newcommand{\edgeset}{\cE}

\newcommand{\node}[2]{(#1,#2)}